\documentclass[letterpaper,11pt]{article}

\usepackage[T1]{fontenc}
\usepackage[utf8]{inputenc}  
\usepackage{fullpage}

\usepackage{hyperref}
\usepackage{url}            
\usepackage{amsmath, amsfonts,amssymb,amsthm}

\usepackage{nicefrac}       
\usepackage{xcolor}         

\usepackage{natbib}

\usepackage{algorithm}
\usepackage{algorithmic}
\usepackage{mathtools}
\DeclareMathOperator*{\argmin}{arg\,min}
\DeclareMathOperator*{\minimize}{minimize}
\newcommand{\eqdef}{\coloneqq}

\newcommand{\sqnorm}[1]{\left\| #1 \right\|^2}
\newcommand{\Exp}[1]{\mathbb{E}\!\left[ #1 \right]}

\usepackage{xspace}
\usepackage{scalefnt}
\definecolor{mydarkred}{rgb}{0.8,0.0,0.0}
\newcommand{\algn}[1]{{\sf\color{mydarkred}\scalefont{0.958}{#1}}\xspace}
\newcommand{\algno}{\algn{TAMUNA}}

\definecolor{mydarkgreen}{rgb}{0,0.55,0}

\newcommand{\mrho}{\tau}
\newcommand{\myc}{\alpha}
\newcommand{\LL}{\mathcal{L}}
\newcommand{\LS}{L}

\usepackage{colortbl}
\definecolor{mylgreen}{rgb}{0.87,1,0.82}
\usepackage{pifont}
\definecolor{mygreen1}{rgb}{0,0.8,0}
\definecolor{gray}{rgb}{0.9,0.9,0.9}
\newcommand{\cmark}{\textcolor{mygreen1}{\ding{51}}}%
\newcommand{\xmark}{\textcolor{red}{\ding{55}}}%
\definecolor{yel}{rgb}{1,0.95,0.8}

\newtheorem{theorem}{Theorem}
\newtheorem{remark}[theorem]{Remark}
\newtheorem{corollary}[theorem]{Corollary}

\title{\textbf{TAMUNA: Doubly Accelerated Distributed Optimization under Partial Participation}}

 \author{Laurent Condat${\!\:}^{1,2}$ \qquad Ivan Agarsk\'y${\!\:}^{3,4}$ \qquad 
 Grigory Malinovsky${\!\:}^1$ \qquad 
 Peter Richt\'{a}rik${\!\:}^{1,2}$\\
 \phantom{xx}
 \\
 ${}^1$Computer Science Program, CEMSE Division,\\
 King Abdullah University of Science and Technology (KAUST)\\ 
 Thuwal, 23955-6900, Kingdom of Saudi Arabia\\[2mm] 
 ${}^2$SDAIA-KAUST Center of Excellence in Data Science\\
 and Artificial Intelligence (SDAIA-KAUST AI)\\[2mm]
${}^3$Brno University of Technology\\ 
  Brno, Czech Republic\\[2mm] 
  ${}^4$Kempelen Institute of Intelligent Technologies (KInIT)\\ 
   Bratislava, Slovakia}

 \date{May 2023. Minor revision in May 2026}
   
\begin{document}

\maketitle

\begin{abstract}
In distributed optimization and federated learning, slow and costly communication between parallel devices and the central server constitutes the primary bottleneck. To alleviate this burden, two strategies have emerged: 1) local training (LT), which reduces communication frequency by performing multiple local computations between rounds, and 2) compression (CC), which consists of transmitting lower-dimensional, compact representations. Recent theoretical advances have successfully combined LT and CC to achieve doubly-accelerated communication rates, with respect to both condition number and model dimension. However, these methods have a major drawback: they require full client participation and break down when idle clients miss communication triggers. We introduce TAMUNA, the first algorithm to successfully intertwine LT, CC, and partial participation. By decoupling primal model updates from dual control variates, TAMUNA overcomes the architectural deadlock of prior methods. In the strongly convex setting, TAMUNA converges linearly to the exact solution, establishing a new state of the art by exhibiting doubly-accelerated convergence, while supporting arbitrary levels of client participation.
\end{abstract}

\newpage 
\tableofcontents
\newpage 

\section{Introduction}

In traditional machine learning, data is first gathered and centralized in a datacenter before processing. By contrast, the modern paradigm of \textbf{Federated Learning (FL)} \citep{kon16a,kon16,mcm17, bon17} trains models collaboratively utilizing the wealth of information stored on edge devices, such as mobile phones and hospital workstations, without ever sharing private data. While this approach addresses critical data privacy and security challenges \citep{kai19,li20,wan21}, it introduces a severe bottleneck: \textbf{communication}. The parallel devices must continually communicate back and forth with a distant orchestrating server over networks that are typically slow, costly, and unreliable. Overcoming this communication bottleneck is the primary hurdle before FL can achieve ubiquitous adoption.

To reduce this communication burden, two distinct strategies have become foundational: 1) \textbf{Local Training (LT)}, which performs multiple local computations between communication rounds to send richer updates, and 2) \textbf{Communication Compression (CC)}, which transmits compact, lower-dimensional representations instead of full vectors. Recently, it was shown that these two mechanisms can be successfully combined to achieve doubly-accelerated communication rates \citep{con26cs}.

However, existing frameworks that successfully integrate LT and CC rely on communication triggers that fundamentally assume full client participation. In practical FL deployments, it is entirely unrealistic to assume all clients are available 100\% of the time. \textbf{Partial participation (PP)} is an essential feature. Introducing PP into a loopless LT architecture is highly non-trivial: idle clients do not perform computations and subsequently miss random communication triggers, which breaks the delicate variance-reduction mechanisms required to prevent client drift.

In this paper, we overcome this architectural barrier with \algno, a novel randomized algorithm for communication-efficient distributed optimization. \algno seamlessly intertwines LT, CC, and PP. By decoupling the primal model updates from the dual control variate updates, \algno provably retains the benefits of all three mechanisms. Being variance-reduced \citep{gor202,gow20a}, 
it converges linearly to an exact solution when exact gradients are used. The main feat is that \algno achieves a \textbf{doubly-accelerated convergence rate}, with respect to both the condition number of the functions and the model dimension, while robustly supporting arbitrary levels of client sampling.

In the remainder of this section, we formulate our setup, propose a comprehensive model to characterize total communication complexity, review the state of the art, and summarize our theoretical contributions.

\subsection{Formalism}

 We consider the standard client-server federated setting, where $n \geq 2$ clients perform local computations in parallel and communicate with a central coordinating server. Our objective is to solve the finite-sum optimization problem:
 \begin{equation}
 \minimize_{x\in\mathbb{R}^d}\  f(x) \eqdef \frac{1}{n}\sum_{i=1}^n f_i(x),\label{eqpro1}
 \end{equation}
 where $f_i:\mathbb{R}^d\rightarrow \mathbb{R}$ represents the individual loss function of client $i\in [n]\eqdef \{1,\ldots,n\}$. Problem \eqref{eqpro1} is fundamental, serving as the standard abstraction for empirical risk minimization, the dominant framework in supervised machine learning. Despite its apparent mathematical simplicity, solving it efficiently is highly challenging in the federated regime, where both the number of clients $n$ and the model dimension $d$ can be massive.

  Every local function $f_i$ is assumed to be $\LL$-smooth and $\mu$-strongly convex for some $\LL \geq \mu > 0$.\footnote{A function $f:\mathbb{R}^d\rightarrow \mathbb{R}$ is $\LL$-smooth if it is differentiable and its gradient is Lipschitz continuous with constant $\LL$; that is, for every $x,y\in\mathbb{R}^d$, $\|\nabla f(x)-\nabla f(y)\|\leq \LL \|x-y\|$. It is $\mu$-strongly convex if $f-\frac{\mu}{2}\|\cdot\|^2$ is convex. We use the Euclidean norm throughout and refer to \cite{bau17} for standard notions of convex analysis.} Due to strong convexity, the sought solution $x^\star$ of \eqref{eqpro1} exists and is unique, and we define the condition number as $\kappa \eqdef \frac{\LL}{\mu}$. We focus primarily on the strongly convex case because analyzing linear convergence rates provides clear mathematical insights into the underlying algorithmic mechanisms of LT, CC, and PP. While extending variance-reduced algorithms to nonconvex objectives 
  remains an active and open area of research, 
 doing so requires fundamentally different proof techniques \cite{kar21,das22} and falls outside the scope of this paper.
   Crucially, our setting accommodates \emph{arbitrary data heterogeneity}: we make no assumptions whatsoever regarding the similarity of the local functions $f_i$ beyond smoothness and strong convexity. 
  
  To solve \eqref{eqpro1}, the baseline algorithm of Gradient Descent (\algn{GD}) iterates the following for $t=0,1,\ldots$:
  \begin{equation*}x^{t+1} \eqdef x^t - \frac{\gamma}{n}\sum_{i=1}^n \nabla f_i(x^t),
  \end{equation*}
  for some stepsize $\gamma \in (0,\frac{2}{\LL})$. Iteration $t$ proceeds in three steps: 1) the server broadcasts $x^t$ to all clients; 2) the clients compute their local gradients $\nabla f_i(x^t)$ and send them to the server in parallel; and 3) the server averages the received vectors to form $\nabla f(x^t)$ and performs the gradient descent step. For $\gamma = \Theta(\frac{1}{\LL})$, \algn{GD} converges linearly, reaching $\epsilon$-accuracy in $\mathcal{O}(\kappa \log\epsilon^{-1})$ iterations. Because full $d$-dimensional vectors are communicated at every iteration, the communication complexity of \algn{GD} is $\mathcal{O}(d\kappa \log\epsilon^{-1})$. Our objective is to achieve a \textbf{twofold acceleration} over \algn{GD}, improving the complexity's dependence on both $\kappa$ and $d$.

\subsection{A model of Asymmetric Communication}\label{secacr}

To rigorously quantify communication efficiency, we distinguish between two directions of information transfer:
\begin{itemize}
\item \textbf{Uplink Communication (UpCom)}: The parallel transmission of locally computed vectors from the clients to the server.
\item \textbf{Downlink Communication (DownCom)}: The broadcast of the aggregated global model from the server to the clients.
\end{itemize}
In federated networks, UpCom is typically the primary bottleneck, much like uploading is slower than downloading on internet or cellular networks. This asymmetry stems from bandwidth limits, network protocols, and the severe computational overhead placed on the server, which must simultaneously decode, process, and aggregate $n$ high-dimensional vectors.

We measure the uplink or downlink \textbf{communication complexity}  as the expected number of communication rounds required to reach $\epsilon$-accuracy, multiplied by the number of real values (floats) transmitted between the server and any single client during one round. For example, the UpCom and DownCom complexity of \algn{GD} is $\mathcal{O}(d \kappa \log\epsilon^{-1})$. While one could measure this in bits, assuming a standard floating-point representation makes this a mere constant factor that does not alter the asymptotic bounds.

Because UpCom is typically slower than DownCom, we evaluate algorithms using the comprehensive \textbf{total communication  (TotalCom)} metric \citep{con26cs}, defined as a weighted sum of the two complexities:
\begin{equation}
\mbox{TotalCom} = \mbox{UpCom} + \myc \cdot \mbox{DownCom},\label{eqtotcom}
\end{equation}
where the weight $\myc \in [0,1]$ dictates the relative cost of downloading. A symmetric, idealized communication regime corresponds to $\myc=1$, while strictly ignoring DownCom (since UpCom is the limiting factor) corresponds to $\myc=0$. In practice, realistic values for $\myc$ are positive but small. By providing explicit TotalCom expressions for any arbitrary $\myc \in [0,1]$, our evaluation model is significantly more robust than literature that exclusively analyzes UpCom ($\myc=0$).

\section{Related Work}\label{secsota}

In this section, we review the literature concerning LT, CC, and PP.

 \subsection{Local Training (LT)}

In standard \algn{GD}, gradient vectors are transmitted immediately after computation. Local Training (LT) introduces the simple yet highly effective strategy of executing multiple local descent steps before communicating with the server. Initially introduced as a heuristic in the popular \algn{FedAvg} algorithm \citep{mcm17}, LT demonstrated immense practical efficiency despite an initial lack of theoretical guarantees. Early analyses of LT focused on the homogeneous (i.i.d.) data regime or relied on bounded gradient diversity assumptions \citep{had19}. Subsequent work extended the analysis to the heterogeneous regime, which is far more representative of FL \citep{sti19, kha20a, li20a,woo20,gor21,gla22}. 
A well-known challenge in this regime is \emph{client drift} \citep{mal20}: if too many local steps are taken, local models drift toward the isolated minimizers of their respective cost functions $f_i$, rather than the global optimum $x^\star$. To combat this, a subsequent generation of variance-reduced LT methods, such as \algn{Scaffold} \citep{kar20}, \algn{S-Local-GD} \citep{gor21}, and \algn{FedLin} \citep{mit21}, introduced control variates to correct the drift. While these methods converge linearly to the exact solution, their communication complexity remains $\mathcal{O}(d\kappa \log \epsilon^{-1})$, offering no asymptotic improvement over \algn{GD}.

The recent development of accelerated LT methods constitutes a theoretical breakthrough. 
\algn{Scaffnew} \citep{mis22} was the first LT-based algorithm to achieve an accelerated communication complexity of $\mathcal{O}(d\sqrt{\kappa} \log \epsilon^{-1})$. By triggering communication randomly with a small probability $p$ at each iteration, the expected number of local steps between rounds becomes $1/p$. Setting $p=1/\sqrt{\kappa}$ yields the optimal condition number dependency \citep{sca19}. \algn{Scaffnew} was later extended to utilize variance-reduced stochastic gradients \citep{mal22,gor202,gow20a} and was analyzed as a special case of the \algn{RandProx} primal--dual framework \citep{con22rp}. Alternatively, \algn{APDA-Inexact} \citep{sad22b} and \algn{5GCS} \citep{mal22b} approached LT via an inner loop computing inexact proximity operators.

\subsection{Partial Participation (PP)}

Partial participation (PP), or client sampling, allows a protocol to progress even when only a subset of clients is active during a given round. Throughout the paper, we denote by $c\in\{2,\ldots,n\}$ the cohort size, representing the number of active clients participating in any given round. Because edge devices frequently drop offline, PP is an indispensable feature for any practical FL deployment. While PP is well understood for standard SGD-type methods \citep{gow19,con22mu}, successfully integrating it with LT has proven historically difficult. For instance, \algn{Scaffold} supports both LT and PP but yields no theoretical communication acceleration. \algn{FedVARP} \citep{jhu22} incorporates both LT and PP, but its analysis is restricted to nonconvex problems under bounded global variance assumptions that do not apply to our setting.

Crucially, the accelerated \algn{Scaffnew} algorithm strictly requires full participation. Prior to our work, \algn{5GCS} \citep{mal22b} stood as the only algorithm enabling both LT and PP while enjoying accelerated communication. However, \algn{5GCS} relies on an indirect, two-level combination (random selection of proximity operators computed inexactly via an inner LT loop) based on \algn{Point-SAGA} \citep{def16}. Moreover, \algn{5GCS} accelerates with respect to $\kappa$ but not $d$, and requires a relatively large number of local steps per round (at least $\mathcal{O}\big(\sqrt{c\kappa/n} \log \kappa\big)$).

\algno, by contrast, relies on intertwined stochastic processes to directly combine LT and PP.

\subsection{Communication Compression (CC)}

Compression further drastically reduces the communication load by applying a (possibly randomized) operator $\mathcal{C}:\mathbb{R}^d \rightarrow \mathbb{R}^d$, enabling a more compact representation of the transmitted vector. A compressor $\mathcal{C}$ is unbiased if $\Exp{\mathcal{C}(x)}=x$ for all $x\in\mathbb{R}^d$; otherwise, it is biased. A popular unbiased compressor is \texttt{rand}-$k$, for some $k\in [d]\eqdef \{1,\ldots,d\}$, which selects $k$ coordinates of $x$ uniformly at random, multiplies them by the scaling factor $d/k$, and sets all remaining coordinates to zero. While the theoretical output $\mathcal{C}(x)$ still formally lives in $\mathbb{R}^d$ for analysis purposes, in practice, only the $k$ non-zero selected elements are actually transmitted, along with a small number of additional bits required to encode their indices. This allows $k$ to be as small as 1, yielding a massive practical compression factor of $d/k$. Another common sparsifying compressor is \texttt{top}-$k$, which retains only the $k$ coordinates with the largest absolute values, and sets the unselected coordinates to zero \citep{bez20}. \texttt{top}-$k$ is deterministic and biased.

The variance-reduced \algn{DIANA} algorithm \cite{mis19d} marked a milestone by proving linear convergence for a broad class of unbiased compressors. Using independent \texttt{rand}-$1$ compressors for UpCom, \algn{DIANA}'s complexity is $\mathcal{O}\big((\kappa(1+\frac{d}{n})+d) \log \epsilon^{-1}\big)$, greatly improving upon \algn{GD} when $n$ is large. \algn{DIANA} has inspired numerous extensions \cite{hor22d,gor202,zli20}, including \algn{DIANA-PP} \cite{con22mu}, which supports PP. However, \algn{DIANA} broadcasts the full, uncompressed model during DownCom, meaning its TotalCom complexity can actually exceed that of \algn{GD}. Bidirectional compression variants exist \cite{gor203,phi20,liu20,con22mu}, but they do not improve the TotalCom complexity. Algorithms using biased compressors, such as \algn{EF21} \cite{ric21,fat21,con22e}, also converge linearly but retain the suboptimal $d\kappa$ dependency.

\subsection{Bridging the Gap: Unifying LT, CC, and PP}

Relying on LT or CC alone is insufficient to fully optimize communication. Recently, \citet{con26cs} successfully merged these two paradigms with \algn{CompressedScaffnew}, the first algorithm to achieve a doubly-accelerated TotalCom rate (improving dependencies on both $\kappa$ and $d$). However, the fundamental limitation of \algn{CompressedScaffnew} is its loopless, randomized communication structure, which inherently mandates $100\%$ client participation. If clients drop out, they miss random communication triggers and fail to update their control variates, instantly breaking the variance reduction mechanism. Therefore, applying \algn{CompressedScaffnew} to real-world FL is impractical. \algno serves as the necessary, mathematically rigorous architectural evolution of \algn{CompressedScaffnew}. By introducing a novel, decoupled two-loop stochastic process, \algno successfully isolates the primal model updates from the dual control variate updates. This structural innovation seamlessly intertwines LT, CC, and PP, explicitly generalizing the double acceleration of \algn{CompressedScaffnew} to arbitrary levels of partial participation. As summarized in Tables~\ref{tab1c} and \ref{tab1}, \algno establishes a new state of the art.

 \begin{table*}[t]
 \centering
 \caption{UpCom complexity ($\myc=0$) of linearly converging algorithms with LT or CC and allowing for PP (with exact gradients). The $\widetilde{\mathcal{O}}$ notation hides the $\log\epsilon^{-1}$ factor (and other log factors for \algn{Scaffold}). 
 $c\in\{2,\ldots,n\}$ is the number of participating clients and the other notations are recalled in Table~\ref{tab2}.
 \label{tab1c}\medskip}
 \begin{tabular}{ccccc}
 Algorithm&LT&CC&UpCom \\
 \hline
 \algn{DIANA-PP}${}^{\!(a)}$&\xmark&\cmark&
 $\widetilde{\mathcal{O}}\!\left((1+\frac{d}{c})\kappa+d\frac{n}{c} \right)$
 \\
  \algn{Scaffold}&\cmark&\xmark&$\widetilde{\mathcal{O}}(d\kappa+d\frac{n}{c})$\\
   \algn{5GCS}&\cmark&\xmark&$\widetilde{\mathcal{O}}\!\left(d\sqrt{\kappa}\sqrt{\frac{n}{c}}+d\frac{n}{c}\right)$\\
   \rowcolor{mylgreen}
  \algn{TAMUNA}&\cmark&\cmark&$\widetilde{\mathcal{O}}\!\left( \sqrt{d}\sqrt{\kappa}\sqrt{\frac{n}{c}}+d\sqrt{\kappa}\frac{\sqrt{n}}{c}+d\frac{n}{c}\right)$\\
 \hline
 \end{tabular}
 \begin{flushleft}
$\qquad$\small $(a)$ using independent \texttt{rand}-1 compressors, for instance. Note that $\mathcal{O}(\sqrt{d}\sqrt{\kappa}\sqrt{\frac{n}{c}}+d\frac{n}{c})$ is better than $\mathcal{O}(\kappa+d\frac{n}{c})$ and $\mathcal{O}(d\sqrt{\kappa}\frac{\sqrt{n}}{c}+d\frac{n}{c})$ is better than $\mathcal{O}(\frac{d}{c}\kappa+d\frac{n}{c})$, 
so that \algno has a better complexity than \algn{DIANA-PP}.
 \end{flushleft}
 \end{table*}

 \begin{table*}[t]
 \centering
 \caption{TotalCom complexity of linearly converging algorithms using Local Training (LT), Communication Compression (CC), or both, in case of full participation and exact gradients. 
 The $\widetilde{\mathcal{O}}$ notation hides the $\log \epsilon^{-1}$ factor. The notations are recalled in Table~\ref{tab2}.
 \label{tab1}\medskip}
 \small
 \begin{tabular}{ccccc}
 Algorithm&LT&CC&TotalCom&TotalCom{=}UpCom when $\myc=0$ \\
 \hline
 \algn{DIANA}${}^{\!(a)}$&\xmark&\cmark&
 $\widetilde{\mathcal{O}}\!\left((1+\myc d+\frac{d+\myc d^2}{n})\kappa+d +\myc d^2\right)$&
 $\widetilde{\mathcal{O}}\!\left((1+\frac{d}{n})\kappa+d \right)$
 \\
   \algn{EF21}${}^{\!(b)}$&\xmark&\cmark&$\widetilde{\mathcal{O}}(d\kappa)$&$\widetilde{\mathcal{O}}(d\kappa)$\\
  \algn{Scaffold}&\cmark&\xmark&$\widetilde{\mathcal{O}}(d\kappa)$&$\widetilde{\mathcal{O}}(d\kappa)$\\
  \algn{FedLin}&\cmark&\xmark&$\widetilde{\mathcal{O}}(d\kappa)$&$\widetilde{\mathcal{O}}(d\kappa)$\\
  \algn{S-Local-GD}&\cmark&\xmark&$\widetilde{\mathcal{O}}(d\kappa)$&$\widetilde{\mathcal{O}}(d\kappa)$\\
 \algn{Scaffnew}&\cmark&\xmark&$\widetilde{\mathcal{O}}(d\sqrt{\kappa})$&$\widetilde{\mathcal{O}}(d\sqrt{\kappa})$\\
  \algn{5GCS}&\cmark&\xmark&$\widetilde{\mathcal{O}}(d\sqrt{\kappa})$&$\widetilde{\mathcal{O}}(d\sqrt{\kappa})$\\
 \algn{FedCOMGATE}&\cmark&\cmark&$\widetilde{\mathcal{O}}(d\kappa)$&$\widetilde{\mathcal{O}}(d\kappa)$\\
 \rowcolor{mylgreen}
\algno&\cmark&\cmark&$\widetilde{\mathcal{O}}\!\left(\sqrt{d}\sqrt{\kappa}+d\frac{\sqrt{\kappa}}{\sqrt{n}}+d+\sqrt{\myc}\,d\sqrt{\kappa}\right)$&$\widetilde{\mathcal{O}}\!\left(\sqrt{d}\sqrt{\kappa}+d\frac{\sqrt{\kappa}}{\sqrt{n}}+d\right)$\\
 \hline
 \end{tabular}
 \begin{flushleft}
$\qquad$\small $(a)$ using independent \texttt{rand}-1 compressors, for instance. Note that $\mathcal{O}(\sqrt{d}\sqrt{\kappa}+d)$ is better than $\mathcal{O}(\kappa+d)$ and $\mathcal{O}(d\frac{\sqrt{\kappa}}{\sqrt{n}}+d)$ is better than $\mathcal{O}(\frac{d}{n}\kappa+d)$, so that \algno has a better complexity than \algn{DIANA}.
 
 \noindent$\qquad$  $(b)$ using \texttt{top}-$k$ compressors with any $k$, for instance.
 \end{flushleft}
 \end{table*}

 \begin{table*}[t]
 \centering
 \caption{Summary of the main notations used in the paper.
 \label{tab2}\medskip}
 \begin{tabular}{cl}
 \hline
 LT & local training\\
 CC & communication compression\\
 PP&partial participation (a.k.a.\ client sampling)\\
 $\LL$&smoothness constant\\
$\mu$&strong convexity constant\\
$\kappa=\LL/\mu$& condition number of the functions\\
$d$ &dimension of the model\\
 $n$, $i$ &number and index of clients\\
$[n]=\{1,\ldots,n\}$ \\
$\myc$& weight on downlink communication (DownCom), see \eqref{eqtotcom}\\
$\sigma_i^2$, $\sigma^2\eqdef \sum_i \sigma_i^2$&variance of the stochastic gradients, see \eqref{eqvarb}\\
$c \in \{2,\ldots,n\}$& number of active clients (a.k.a.\ cohort size). Full participation if $c=n$\\
$\Omega \subset [n]$& index set of active clients\\
$s\in \{2,\ldots,c\}$&sparsity index for compression. No compression if $s=c$\\
$\mathbf{q}=(q_i)_{i=1}^c$&random binary mask for compression\\ 
$r$ &index of rounds\\
$\LS$, $\ell$ &number and index of local steps in a round\\
$p$& inverse of the expected number of local steps per round\\
$t$, $T$ &indexes of iterations\\
$\gamma$, $\eta$, $\chi$ &stepsizes\\
$x_i$ &local model estimate at client $i$\\
$h_i$ &local control variate tracking $\nabla f_i$\\
$\bar{x}^{(r)}$&model estimate at the server at round $r$\\
$\mrho$&convergence rate\\
 \hline
 \end{tabular}
 \end{table*}
 
 \section{Challenges and Contributions}\label{seccc}
 
 Let us examine the mathematical friction inherent in combining LT with both PP and CC, and outline our methodological contributions. Our notations are summarized in Table~\ref{tab2} for convenience.
 
 \subsection{The Friction Between LT and PP: Decoupling Primal and Dual Updates}
 
 With the breakthrough of \algn{Scaffnew} \citep{mis22}, LT is now understood to yield theoretical communication acceleration. However, integrating PP into this randomized framework is notoriously difficult; according to \citet{mal22b}, the authors of \algn{Scaffnew} \emph{``have tried---very hard in their own words---but their efforts did not bear any fruit.''} The challenge lies in the fact that PP must apply not only to the communication step itself, but to all preceding local computations. The naive approach—allowing active clients to proceed normally while idle clients freeze their local variables—destroys the variance reduction mechanism. \algno successfully overcomes this by decoupling the primal model update from the dual control variate update. A key design property of \algno is that \emph{only} the clients that actively participated in a given round utilize the newly broadcast global model to update their control variates (step 14). 
 By treating probabilistic communication and random client selection as two intertwined but distinct stochastic processes, \algno ensures that the sum of the control variates across all clients remains exactly zero at all times. This decoupling allows \algno to fully retain the $\mathcal{O}(\sqrt{\kappa})$ acceleration of LT, regardless of the participation level.
 
 \subsection{The Friction Between LT and CC: A Novel Two-Loop Structure}
 Combining LT and CC introduces a second severe challenge: controlling the variance of compression errors alongside the variance of client drift. Simply applying standard compressors (like \algn{DIANA} or \algn{EF21}) to \algn{Scaffnew} fails. In standard algorithms, the differences between local gradients and control variates are compressed, which works because gradients are evaluated at a shared global model. In an LT setting, local models drift apart, rendering these gradient differences unusable. Furthermore, compressing the difference between the drifted local model and the \emph{last known} global model ruins the benefits of LT because the global reference is too stale. Therefore, the local model estimates themselves must be compressed. While \algn{CompressedScaffnew} \citep{con26cs} achieves this using a permutation-based compression mechanism,  
 its loopless architecture strictly requires full participation. Non-participating clients perform no computation, missing the random communication triggers entirely. \algno resolves this architectural deadlock by introducing a novel \textbf{two-loop structure}. Unlike loopless algorithms where communication can trigger at any iteration, \algno explicitly defines a communication ``round'' as a sequence of local steps followed by synchronized, compressed communication. Client participation is determined at the \emph{beginning} of the round. This structure accommodates idle clients naturally, allowing \algno to function with any arbitrary level of PP (as few as $c=2$ clients per round) while safely utilizing the permutation-based compressors of \algn{CompressedScaffnew}. It remains an open question whether other types of compressors can be used in \algno, such as applying quantization on top of sparsification \citep{hor22, alb20}. Ultimately, \algno establishes the new state of the art for communication-efficient FL. For instance, using exact gradients under full participation, if $\myc$ is small and $n$ is large, its TotalCom complexity is
 \begin{equation*}\mathcal{O}\!\left(\left(\sqrt{d}\sqrt{\kappa}+d\right)\log \epsilon^{-1}\right),
 \end{equation*}
 demonstrating a simultaneous twofold acceleration: $\sqrt{\kappa}$ instead of $\kappa$ (via LT), and $\sqrt{d}$ instead of $d$ (via CC). Our general convergence result is formalized in Theorem~\ref{theo2}.

\begin{algorithm}[t]
	\caption{\algno}
	\label{alg0}
	\begin{algorithmic}[1]
		\STATE \textbf{input:}  stepsizes $\gamma>0$, $\eta >0$;  
		number of participating clients $c\in \{2,\ldots,n\}$; sparsity index for compression $s\in \{2,\ldots,c\}$;
		initial model estimate $\bar{x}^{(0)} \in \mathbb{R}^d$ at the server and initial control variates $h_1^{(0)}, \ldots, h_n^{(0)} \in \mathbb{R}^d$ at the clients, such that $\sum_{i=1}^n h_i^{(0)}=0$.
		\FOR{$r=0,1,\ldots$ (rounds)}
		\STATE choose a subset $\Omega^{(r)} \subset [n]$ of size 
		$c$
		 uniformly at random
		 \STATE choose the number of local steps $\LS^{(r)}\geq 1$ 
		\FOR{clients $i\in \Omega^{(r)}$, in parallel,}{}
	\STATE $x_i^{(r,0)}\eqdef \bar{x}^{(r)}$ (initialization received from the server)
	 \FOR{$\ell=0,\ldots,\LS^{(r)}-1$ (local steps)}{}
\STATE $x_i^{(r,\ell+1)}\eqdef x_i^{(r,\ell)} -\gamma g_i^{(r,\ell)}+\gamma h_i^{(r)}$, 
where $g_i^{(r,\ell)}$ is an unbiased stochastic estimate of $\nabla f_i\big(x_i^{(r,\ell)}\big)$ of variance $\sigma_i^2$
\ENDFOR
\ENDFOR
\STATE 
UpCom: 
the server and active clients agree on a random binary mask $\mathbf{q}^{(r)}=\big(q_i^{(r)}\big)_{i\in\Omega^{(r)}} \in \mathbb{R}^{d \times c}$,
and every client $i \in \Omega^{(r)}$ sends the compressed vector $\mathcal{C}_i^{(r)}\!\left(x_i^{(r,\LS^{(r)})}\right)$ to the server, where $\mathcal{C}^{(r)}_i(v)$ denotes $v$ multiplied elementwise by $q_i^{(r)}$.
\STATE $\bar{x}^{(r+1)}\eqdef \frac{1}{s}\sum_{i\in \Omega^{(r)}}\mathcal{C}_i^{(r)}\!\left(x_i^{(r,\LS^{(r)})}\right)$ (aggregation by the server)
\FOR{clients $i\in \Omega^{(r)}$, in parallel,}{}
\STATE $h_i^{(r+1)}\eqdef h_i^{(r)} + \frac{\eta}{\gamma}\Big(\mathcal{C}_i^{(r)}\!\left(\bar{x}^{(r+1)}\right)-\mathcal{C}_i^{(r)}\!\left(x_i^{(r,\LS^{(r)})}\right)\Big)$ \big($\bar{x}^{(r+1)}$ is received from the server\big)
\ENDFOR
\FOR{clients $i\notin \Omega^{(r)}$, in parallel,}{}
\STATE $h_i^{(r+1)}\eqdef h_i^{(r)}$ (the client is idle)
\ENDFOR
		\ENDFOR
	\end{algorithmic}
\end{algorithm}

 \section{Proposed Algorithm \algno}\label{sectamu}

The proposed algorithm \algno is outlined in Algorithm~\ref{alg0}. Its main loop iterates over communication rounds, indexed by $r$. Each round consists of an inner loop of local steps, indexed by $\ell$ and performed in parallel by the active clients, followed by compressed uplink communication with the server and a subsequent update of the local control variates $h_i$. The $c$ active (participating) clients are selected uniformly at random at the beginning of each round.

During UpCom, every active client sends a compressed version of its local model $x_i$, transmitting only a sparse subset of its coordinates. Specifically, the server and the active clients systematically agree upon a random binary mask $\mathbf{q}^{(r)}=(q_i^{(r)})_{i=1}^c \in \mathbb{R}^{d \times c}$. This mask is generated on the fly by randomly permuting the columns of a fixed binary template pattern designed to contain exactly $s \ge 2$ ones in every row. Consequently, each column vector $q_i^{(r)}$ dictates which coordinates client $i$ transmits. In effect, every client applies a \texttt{rand}-$k$ compressor, where $k \approx \frac{sd}{c}$. However, unlike standard approaches, these compressors are not independent. On the contrary, they are anti-correlated; the permutation-based design ensures that the sparse, compressed messages sent by the participating clients tightly complement one another. This structural dependency maintains a strict control over the variance after aggregation. We refer to \citet{con26cs} for a detailed illustration and the explicit construction rules of this specific compression mechanism.

At the end of the round, the server aggregates the received sparse vectors to form the global model estimate $\bar{x}^{(r+1)}$. This global model is then broadcast exclusively to the active clients, which use it to update their control variates $h_i$. Crucially, this update overwrites only the coordinates of $h_i$ that were actively involved in the communication process (i.e., the indices where the mask $q_i^{(r)}$ equals one). Indeed, the received vector $\bar{x}^{(r+1)}$ lacks the relevant aggregated information necessary to update $h_i$ at the uncommunicated coordinates.

The local model estimates $x_i$ at the clients are updated at the beginning of the round, when the active clients download the current global model estimate $\bar{x}^{(r)}$ to initialize their local steps. While the algorithm is written sequentially showing two DownCom steps from the server to the clients per round (steps 6 and 14), 
in practice, only one is required: $\bar{x}^{(r+1)}$ can be broadcast by the server at the end of round $r$ simultaneously to the active clients of round $r$ (to update their control variates) and the newly selected active clients of round $r+1$ (to initialize their local steps). We maintain the current sequential notation strictly for algorithmic clarity.

Consequently, clients indexed by $i\notin\Omega^{(r)}$ that do not participate in round $r$ remain completely idle. They perform no computations, initiate no communications, and their local control variates $h_i$ remain unchanged. 
Furthermore, idle clients are not required to store a local model estimate; they simply await the reception of the latest global model $x^{(r)}$ from the server upon their next activation.

In \algno, the local steps may utilize unbiased stochastic gradient estimates of bounded variance $\sigma_i^2$. That is, for every $i\in[n]$,
\begin{equation}
\Exp{g_i^{(r,\ell)}\;|\;x_i^{(r,\ell)}}=\nabla f_i\big(x_i^{(r,\ell)}\big),\quad \Exp{\sqnorm{g_i^{(r,\ell)}-\nabla f_i\big(x_i^{(r,\ell)}\big)}\;|\;x_i^{(r,\ell)}}\leq \sigma_i^2,\label{eqvarb}
\end{equation}
for some $\sigma_i\geq 0$ (if exact gradients are used, $\sigma_i=0$ and $g_i^{(r,\ell)}=\nabla f_i\big(x_i^{(r,\ell)}\big)$). We define the total variance as $\sigma^2\eqdef \sum_{i=1}^n \sigma_i^2$.

Our main theoretical result, establishing the linear convergence of \algno to the exact solution $x^\star$ of \eqref{eqpro1} (or to a neighborhood if $\sigma>0$), is stated below.

\begin{theorem}[fast linear convergence to a $\sigma^2$-neighborhood]\label{theo1}
Let $p\in (0,1]$. In \algno, suppose that at every round $r\geq 0$, $\LS^{(r)}$ is chosen randomly and independently according to a geometric law of mean $p^{-1}$; that is, for every $\LS\geq 1$,
$\mathrm{Prob}(\LS^{(r)}=\LS)=(1-p)^{\LS-1}p$. 
Also, suppose that
\begin{equation}
0<\gamma < \frac{2}{\LL}\label{eqc11}
\end{equation}
and $\eta \eqdef p\chi$, where
\begin{equation}
0<\chi\leq   \frac{n(s-1)}{s(n-1)} \in \left(\frac{1}{2},1\right].\label{eqc12}
\end{equation}
For every total number $t\geq 0$ of local steps made so far, define the Lyapunov function
\begin{equation}
\overline{\Psi}^t\eqdef  \frac{n}{\gamma}\sqnorm{\bar{x}^t-x^\star}+ \frac{\gamma }{p^2\chi}
\frac{n-1}{s-1}
\sum_{i=1}^n\sqnorm{h_i^{(r)}-h_i^\star},\label{eqlya1jf}
\end{equation}
where $x^\star$ is the unique solution to \eqref{eqpro1}, $h_i^\star = \nabla f_i(x^\star)$, 
$r\geq 0$ and $\ell\in \{0,\ldots,\LS^{(r)}-1\}$ are such that
\begin{equation*}
t = \sum_{\hat{r}=0}^{r-1} \LS^{(\hat{r})} + \ell, 
\end{equation*}
and
\begin{equation}
\bar{x}^t \eqdef \frac{1}{s}\sum_{i\in \Omega^{(r)}} \mathcal{C}_i^{(r)}\!\left(x_i^{(r,\ell)}\right).\label{fgsdsg}
\end{equation}
Then, for every $t\geq 0$,
\begin{equation}
\Exp{\overline{\Psi}^{t}}\leq \mrho^{t} \overline{\Psi}^0 + \frac{\gamma\sigma^2}{1-\mrho},\label{eqgrg}
\end{equation}
where 
\begin{equation*}
\mrho\eqdef  \max\left((1-\gamma\mu)^2,(\gamma \LL-1)^2,1-  p^2\chi \frac{s-1}{n-1}
\right)<1.
\end{equation*}
Also, if $\sigma=0$, $(\bar{x}^{(r)})_{r\ge 0}$ converges  to $x^\star$ and $(h_i^{(r)})_{r\ge 0}$ converges to $h_i^\star$, almost surely.
\end{theorem}

The complete proof is deferred to the Appendix. We provide a brief sketch here: we first analyze Algorithm~\ref{alg1}, a single-loop equivalent indexed by $t$ with one local step per iteration. In this variant, communication is triggered randomly with probability $p$ (akin to \algn{Scaffnew} and \algn{CompressedScaffnew}), computations occur globally, and PP only governs the communication step. We subsequently demonstrate how Theorem~\ref{theo3} regarding Algorithm~\ref{alg1} translates to the convergence guarantees for the two-loop \algno in Theorem~\ref{theo1}. Because the Lyapunov function strictly contracts per local iteration rather than per random-sized round, we meticulously re-index the local steps to express the rate as a function of $t$.

Note that in \eqref{fgsdsg}, $\bar{x}^t$ is practically computed only when $\ell=0$, yielding $\bar{x}^t=\bar{x}^{(r)}$. Furthermore, Theorem~\ref{theo1} depends explicitly on the sparsity index $s$ but not on the cohort size $c$. The dependence on $c$ is implicitly bounded by the requirement that $s \le c$.

\begin{remark}[setting $\eta$]
Under the conditions of Theorem~\ref{theo1}, one can safely set $\eta = \frac{p}{2}$ in \algno, rendering it independent of $n$ and $s$. However, to maximize the contraction rate, it is recommended to set
\begin{equation}
\eta=  p\frac{n(s-1)}{s(n-1)}.\label{eqeta}
\end{equation}
Also, as a practical heuristic, if $\LS$ is the average number of local steps per round, $p$ can be replaced by $\LS^{-1}$.
\end{remark}

It is worth examining the architectural difference between \algno and \algn{Scaffold} when CC is disabled ($s=c$). In \algno, the control variate $h_i$ is updated using the difference between the \emph{latest} global estimate $\bar{x}^{(r+1)}$ and the latest local estimate $x_i^{(r,\LS^{(r)})}$. Conversely, \algn{Scaffold} utilizes $\bar{x}^{(r)}-x_i^{(r,\LS^{(r)})}$, inherently relying on the ``stale'' global estimate $\bar{x}^{(r)}$. Furthermore, \algn{Scaffold} scales this difference by the number of local steps, artificially shrinking the update size. This structural limitation explains why \algn{Scaffold} intrinsically fails to extract any theoretical communication acceleration from LT, regardless of the number of local steps taken.

Finally, we note that the asymptotic neighborhood size in \eqref{eqgrg} does not exhibit linear speedup (i.e., it does not diminish as $n$ increases). The precise dynamics of LT combined with stochastic gradients remain poorly understood \citep{woo20}. We hypothesize this issue should be analyzed under the broader umbrella of variance reduction \citep{mal22}, which falls outside the scope of this paper's primary focus on communication efficiency.

\subsection{Iteration and Communication Complexities}

In this section, we assume access to exact gradients ($\sigma=0$).\footnote{If $\sigma>0$, one can derive sublinear rates to reach an $\epsilon$-neighborhood by setting $\gamma$ proportional to $\epsilon$, analogous to the approach taken for \algn{Scaffnew} in \citet[Corollary 5.6]{mis22}.} We focus on this deterministic setting to clearly isolate the fundamental algorithmic mechanics and strictly establish the new state of the art for communication complexity, regardless of the underlying local computation noise. We place ourselves under the conditions of Theorem~\ref{theo1}.

We first note that \algno perfectly matches the iteration complexity of standard \algn{GD}, achieving the rate $\mrho^\sharp\eqdef \max(1-\gamma\mu,\gamma \LL-1)^2$, provided $p$ and $s$ are large enough to satisfy $1- \chi p^2 \frac{s-1}{n-1} \leq \mrho^\sharp$. This demonstrates a highly desirable property: incorporating LT, CC, and PP does not harm the baseline convergence rate at all, up to a specific threshold.

Let us formalize the number of iterations (i.e., the total number of local steps) required to reach $\epsilon$-accuracy, defined as $\Exp{\overline{\Psi}^{t}}\leq \epsilon$. For any $s\geq 2$, $p\in(0,1]$, $\gamma =\Theta(\frac{1}{\LL})$, and $\chi =\Theta(1)$, the iteration complexity of \algno is
$\mathcal{O}\left(\left(\kappa + \frac{n }{s p^2}\right)\log \epsilon^{-1}\right)$. 
By optimizing the communication probability $p$ as
\begin{equation}
p = \min\left(\Theta\left(\sqrt{\frac{n}{s\kappa}}\right),1\right),\label{eqc2a1}
\end{equation}
which implies that the expected number of local steps per round is
\begin{equation}
\Exp{\LS^{(r)}}=\max\left(\Theta\left(\sqrt{\frac{s\kappa}{n}}\right),1\right),\label{eqavls1}
\end{equation}
the iteration complexity simplifies to
\begin{equation*}
\mathcal{O}\left(\left(\kappa + \frac{n}{s}\right)\log\epsilon^{-1}\right).
\end{equation*}

We now turn our attention to the communication complexity. Communication is triggered at each iteration with probability $p$. During every communication round, DownCom broadcasts the full $d$-dimensional global vector $\bar{x}^{(r)}$. In contrast, UpCom strictly applies compression, meaning the number of floats sent in parallel by the active clients equals the number of ones per column in the sampling mask $\mathbf{q}$, which is at most $\lceil \frac{sd}{c}\rceil \geq 1$. Consequently, the respective communication complexities are:
\begin{align*}
\mbox{DownCom:}&\quad \mathcal{O}\left(pd\left(\kappa + \frac{n}{s p^2}\right)\log\epsilon^{-1}\right),\\
\mbox{UpCom:}&\quad \mathcal{O}\left(p\left(\frac{sd}{c}+1\right)\left(\kappa + \frac{n}{s p^2}\right)\log\epsilon^{-1}\right),\\
\mbox{TotalCom:}&\quad  \mathcal{O}\left(p\left(\frac{sd}{c}+1+\myc d\right)\left(\kappa + \frac{n}{s p^2}\right)\log\epsilon^{-1}\right).
\end{align*}
For any given sparsity level $s$, the optimal choice for $p$ to minimize both UpCom and DownCom remains \eqref{eqc2a1}. Substituting this yields 
\begin{equation*}
\mathcal{O}\left(p\left(\kappa + \frac{n}{s p^2}\right)\right)= \mathcal{O}\left(\sqrt{\frac{n\kappa}{s}}+\frac{n}{s}\right),
\end{equation*}
resulting in a TotalCom complexity of:
\begin{equation*}
\mbox{TotalCom:}\quad \mathcal{O}\left(\!\left(\sqrt{\frac{n\kappa}{s}}+\frac{n}{s}\right)\!\left(\frac{sd}{c}+1+\myc d\right)\log\epsilon^{-1}\!\right).
\end{equation*}
Here, we clearly observe the first acceleration effect driven by LT: by utilizing a suitable $p<1$, the communication complexity depends only on $\sqrt{\kappa}$, rather than $\kappa$, irrespective of the participation level $c$ or compression level $s$.

If compression is disabled ($s=c$), the TotalCom complexity becomes 
\begin{equation*}
\mathcal{O}\left(d\left(\sqrt{\frac{n\kappa}{c}}+\frac{n}{c}\right)\log\epsilon^{-1}\!\right).
\end{equation*}
However, by dynamically tuning $s$, we can trigger a second acceleration effect, minimizing the TotalCom complexity to establish our primary theoretical result:

\begin{theorem}[doubly accelerated communication]\label{theo2}
In the conditions of Theorem~\ref{theo1}, suppose that $\sigma=0$, $\gamma =\Theta(\frac{1}{\LL})$, $\chi =\Theta(1)$,
and the algorithm parameters are set as
\begin{equation}
p = \min\left(\Theta\left(\sqrt{\frac{n}{s\kappa}}\right),1\right),\quad s = \max\left(2,\left\lfloor \frac{c}{d}\right\rfloor,\lfloor \myc c\rfloor \right).\label{eqs}
\end{equation}
Then, the TotalCom complexity of \algno is:
\begin{equation}
\colorbox{yel}{$\displaystyle
\mathcal{O}\!\left(\left(
\sqrt{d}\sqrt{\kappa}\sqrt{\frac{n}{c}}+d\sqrt{\kappa}\frac{\sqrt{n}}{c}+d\frac{n}{c}+\sqrt{\myc}\,d\sqrt{\kappa}\sqrt{\frac{n}{c}}
\right)
\log\epsilon^{-1}\right)$}.\label{eqbt}
\end{equation}
\end{theorem}

As summarized in Tables \ref{tab1c} and \ref{tab1}, this doubly-accelerated complexity establishes a new state of the art for distributed optimization. The precise behavior of this bound under varying network asymmetries and participation regimes is detailed in the following corollaries.

\begin{corollary}[dependence on $\myc$]
As long as the downlink penalty is sufficiently small, specifically $\myc\leq \max(\frac{2}{c},\frac{1}{d},\frac{n}{\kappa c})$, the TotalCom complexity is dominated by UpCom, identical to the $\myc=0$ case:
\begin{equation*}
\mathcal{O}\!\left(\left(\sqrt{d}\sqrt{\kappa}\sqrt{\frac{n}{c}}+d\sqrt{\kappa}\frac{\sqrt{n}}{c}+d\frac{n}{c}\right)\log\epsilon^{-1}\right).
\end{equation*}
Conversely, if $\myc\geq \max(\frac{2}{c},\frac{1}{d},\frac{n}{\kappa c})$, the downlink cost becomes the bottleneck and the complexity becomes
\begin{equation*}
\mathcal{O}\!\left(\sqrt{\myc}d\sqrt{\kappa}\sqrt{\frac{n}{c}}\log\epsilon^{-1}\right).
\end{equation*}
Crucially, compression remains effective in this regime, evidenced by the dependence on $\sqrt{\myc}$. 
It is only in the perfectly symmetric, uncompressed regime ($\myc=1 \implies s=c$) that the UpCom, DownCom, and TotalCom complexities converge to
\begin{equation*}
\mathcal{O}\!\left(d\sqrt{\kappa}\sqrt{\frac{n}{c}}\log\epsilon^{-1}\right).
\end{equation*}
Thus, even under full participation ($c=n$), \algno systematically outperforms \algn{Scaffnew} for every asymmetry weight $\myc\in [0,1]$.
\end{corollary}

\begin{corollary}[full participation] 
In the limiting case of full client participation ($c=n)$, the TotalCom complexity of \algno simplifies to
\begin{equation*}
\mathcal{O}\!\left(\left(\sqrt{d}\sqrt{\kappa}+d\frac{\sqrt{\kappa}}{\sqrt{n}}+d+\sqrt{\myc}\,d\sqrt{\kappa}\right)\log\epsilon^{-1}\right).
\end{equation*}
\end{corollary}

\begin{figure*}[t]
	\centering
	\hspace*{-1.7mm}\begin{tabular}{cc}
	\includegraphics[scale=0.21]{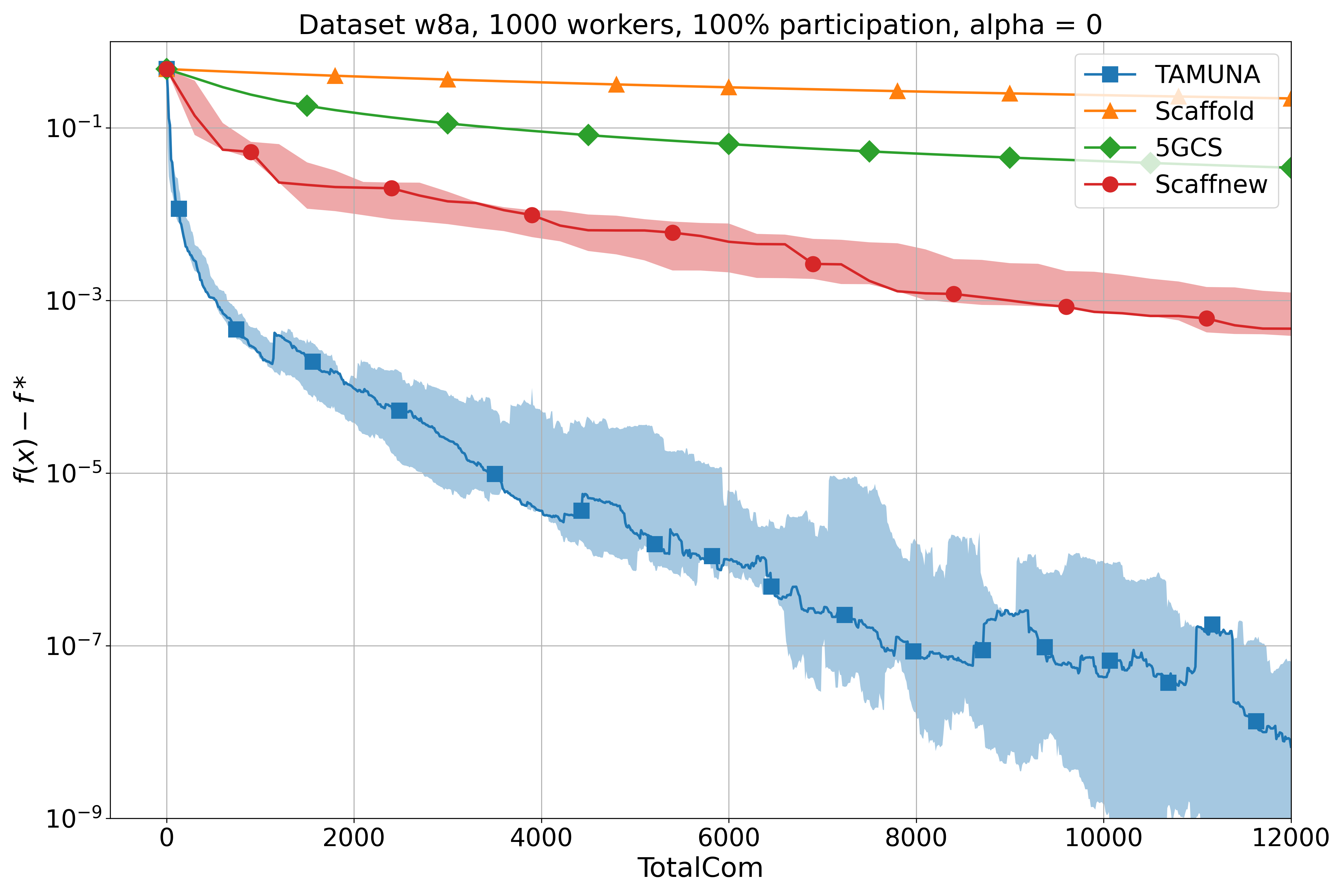}&\!\!\!\!\!\includegraphics[scale=0.21]{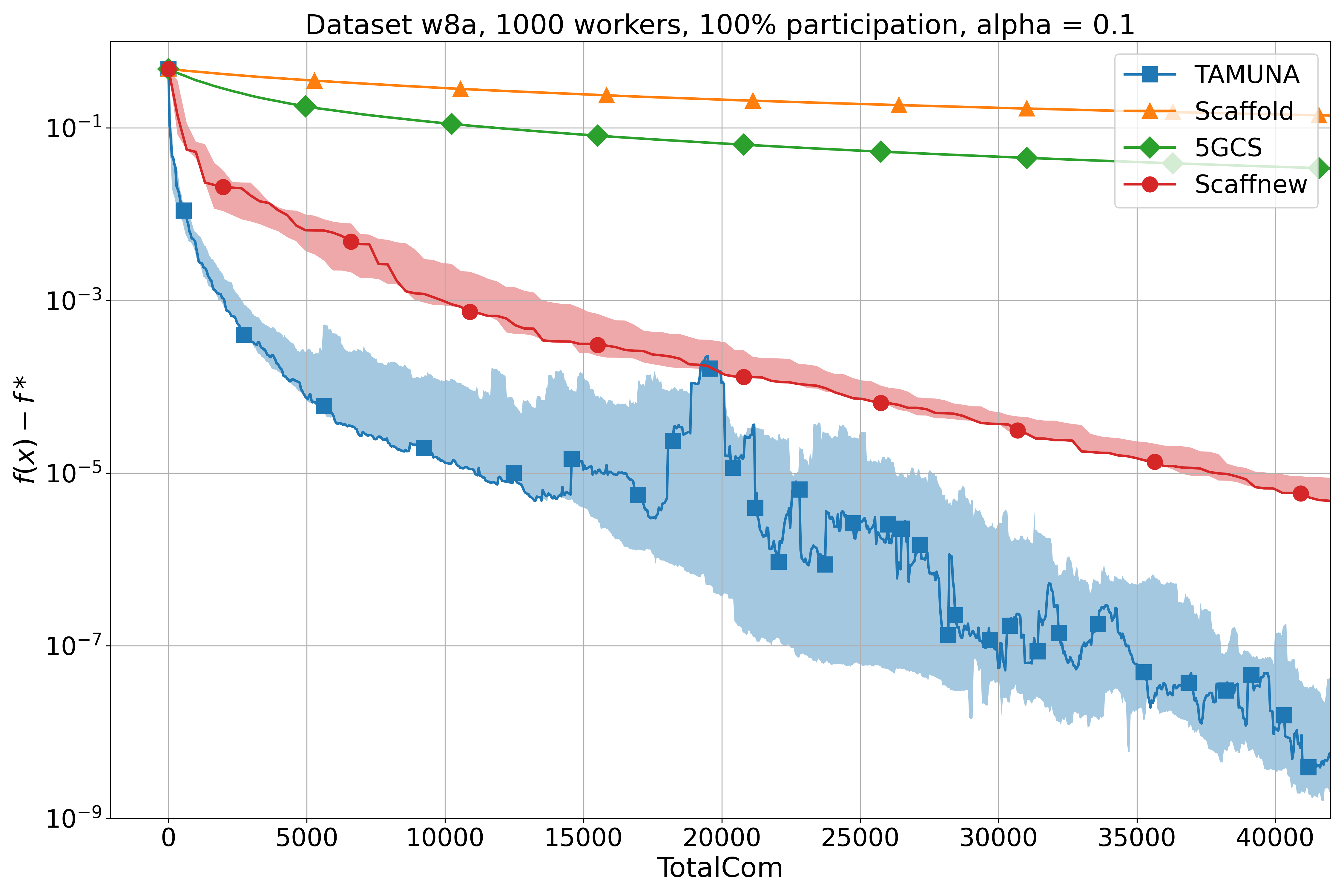}\\
	(a) w8a, $n=1000$, $\alpha=0$, $c=n$&(b) w8a, $n=1000$, $\alpha = 0.1$, $c=n$\\
	\includegraphics[scale=0.21]{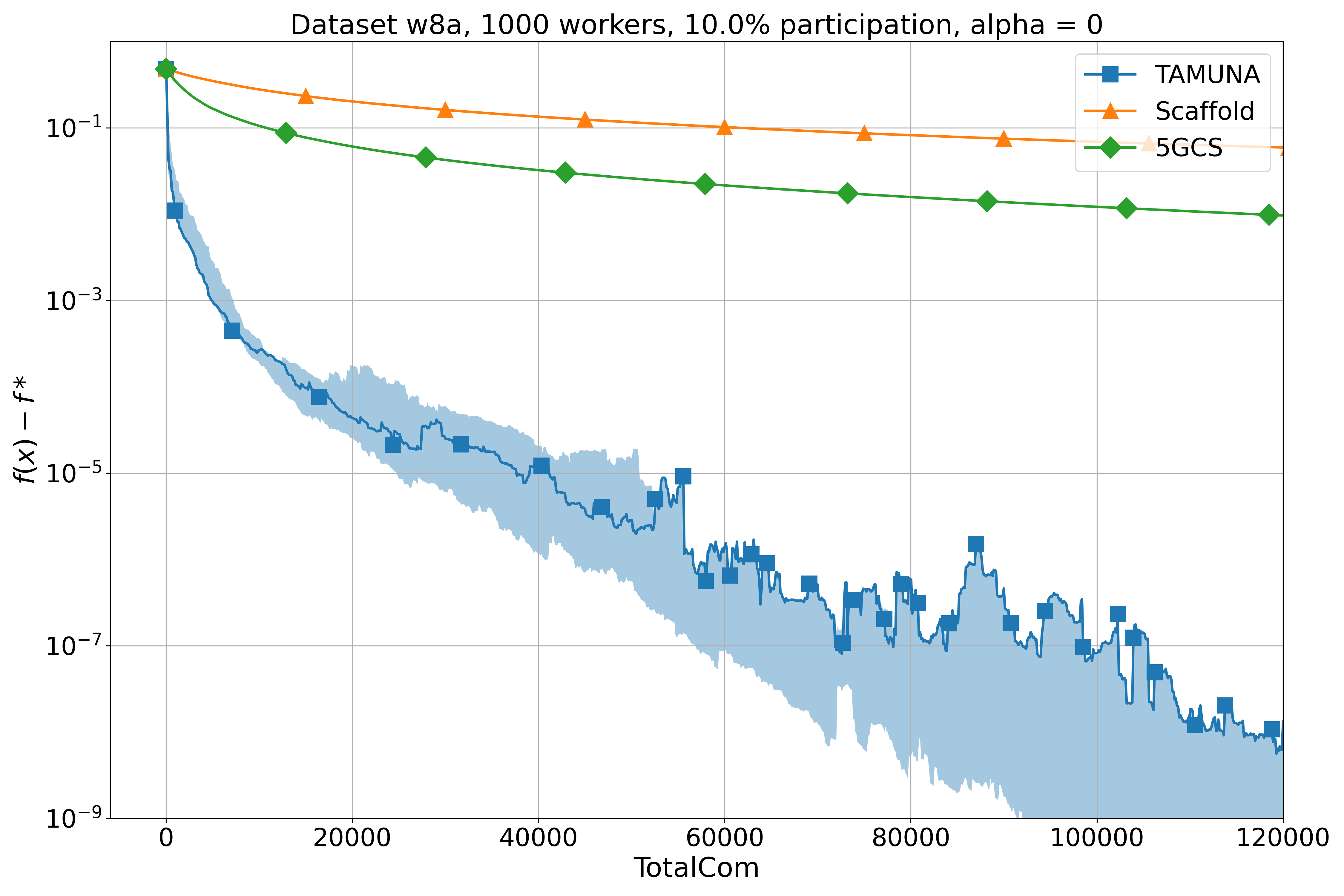}&\!\!\!\!\!\includegraphics[scale=0.21]{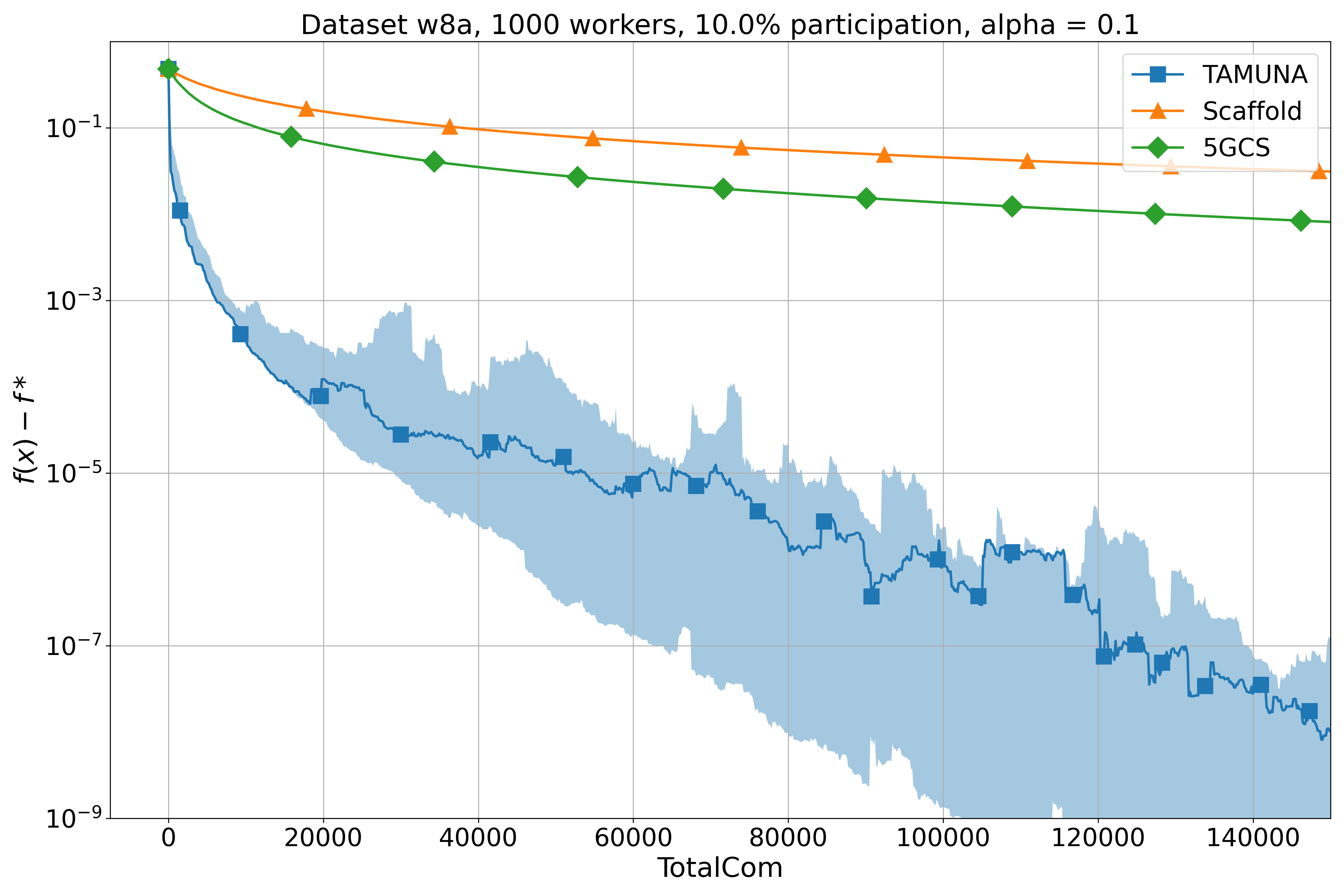}\\
	(c) w8a, $n=1000$, $\alpha=0$, $c=0.1n$&(d) w8a, $n=1000$, $\alpha=0.1$, $c=0.1n$
	\end{tabular}
	\caption{Logistic regression experiment in the case $n>d$. The dataset w8a has $d=300$ features and $n=1000$, so $n \approx 3d$. The first row shows a comparison in the full participation regime, while the second row shows a comparison in the partial participation regime with 10\% of clients. On the left, $\myc=0$, while on the right, $\myc=0.1$.
	\label{fig:totalcostw8a}}
\end{figure*}

\begin{figure*}[t]
	\centering
	\hspace*{-1.5mm}\begin{tabular}{cc}
		\includegraphics[scale=0.21]{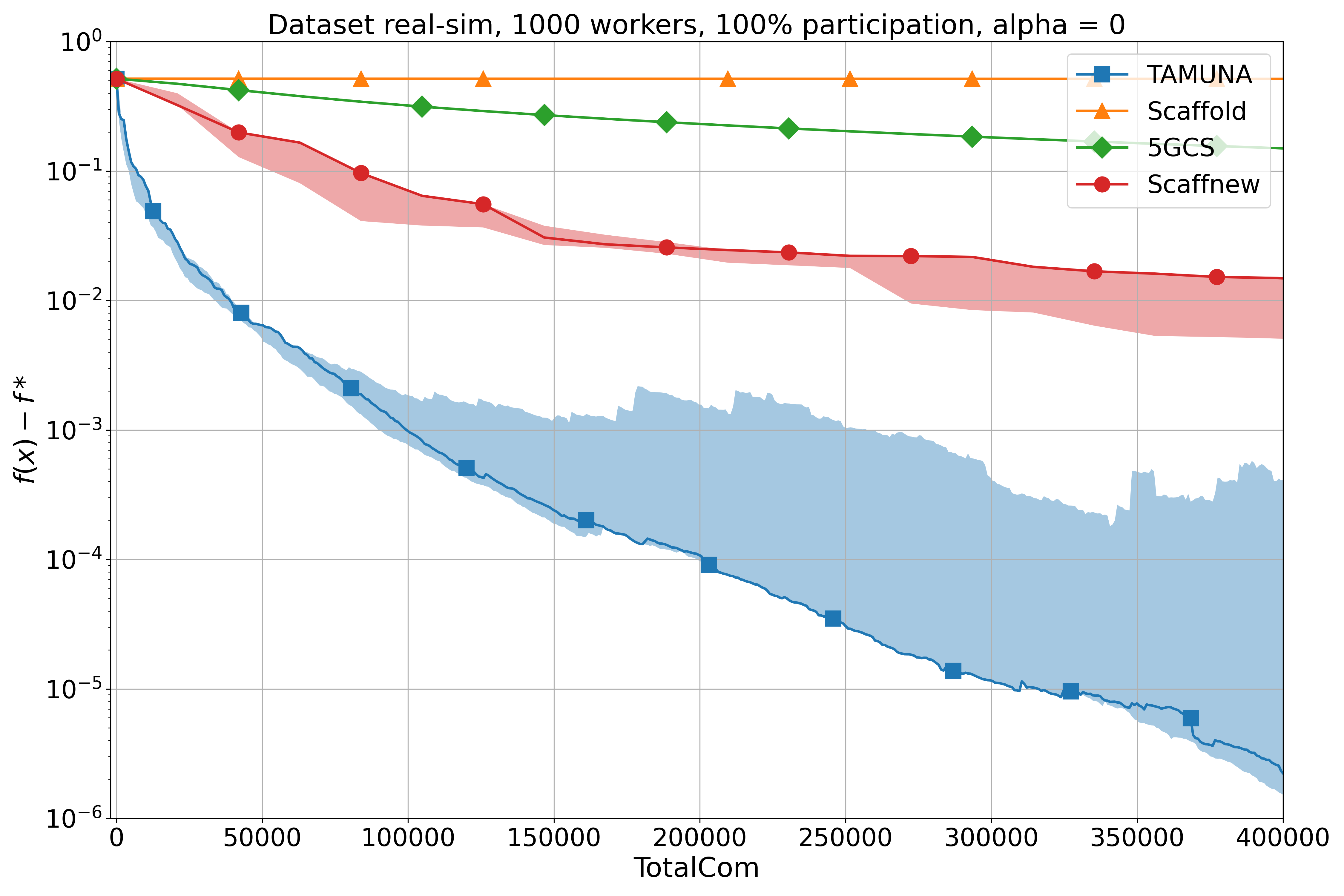}&\!\!\!\!\includegraphics[scale=0.21]{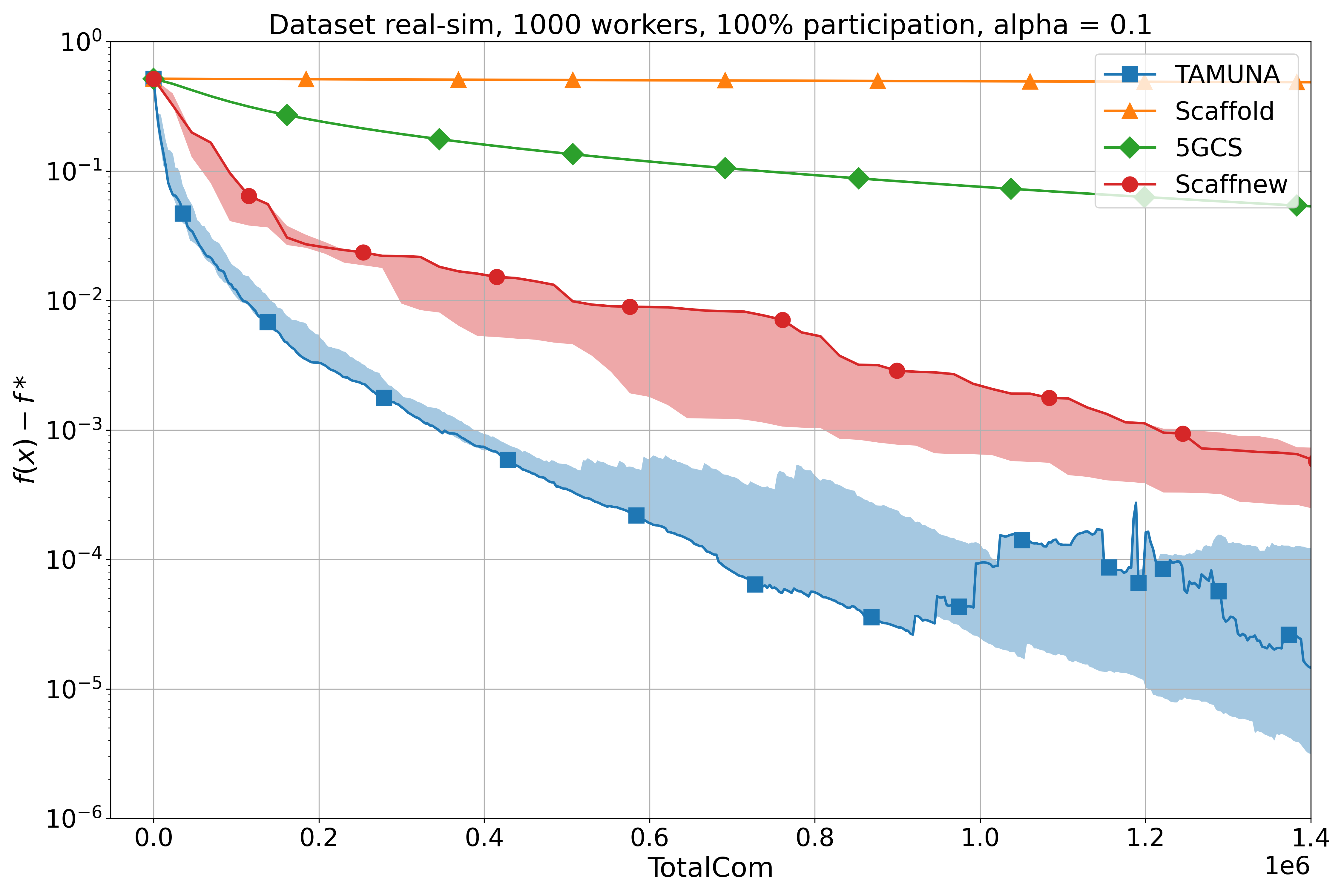}\\
		(a) real-sim, $n=1000$, $\alpha=0$, $c=n$&(b) real-sim, $n=1000$, $\alpha = 0.1$, $c=n$\\
		\includegraphics[scale=0.21]{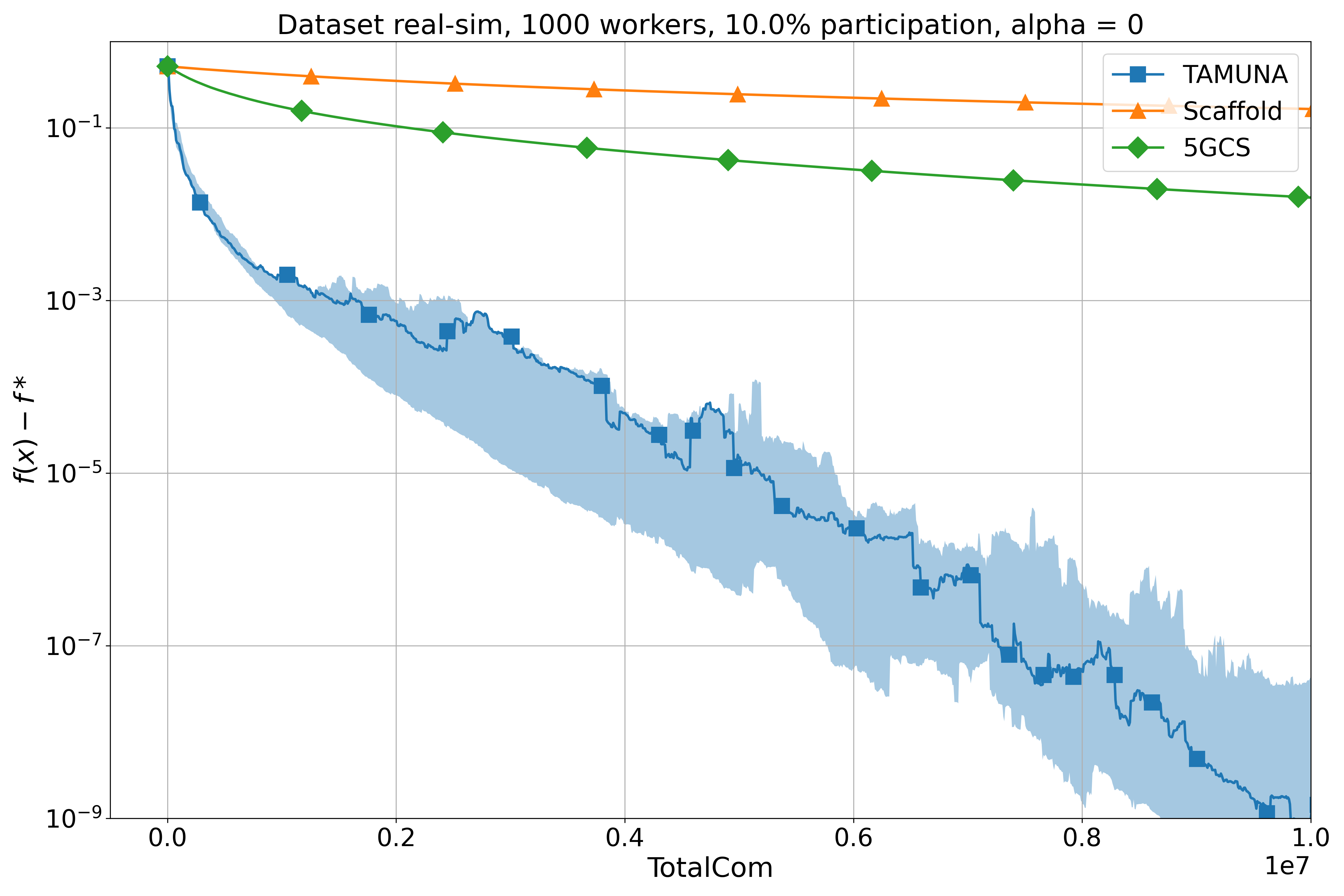}&\!\!\!\!\!\includegraphics[scale=0.21]{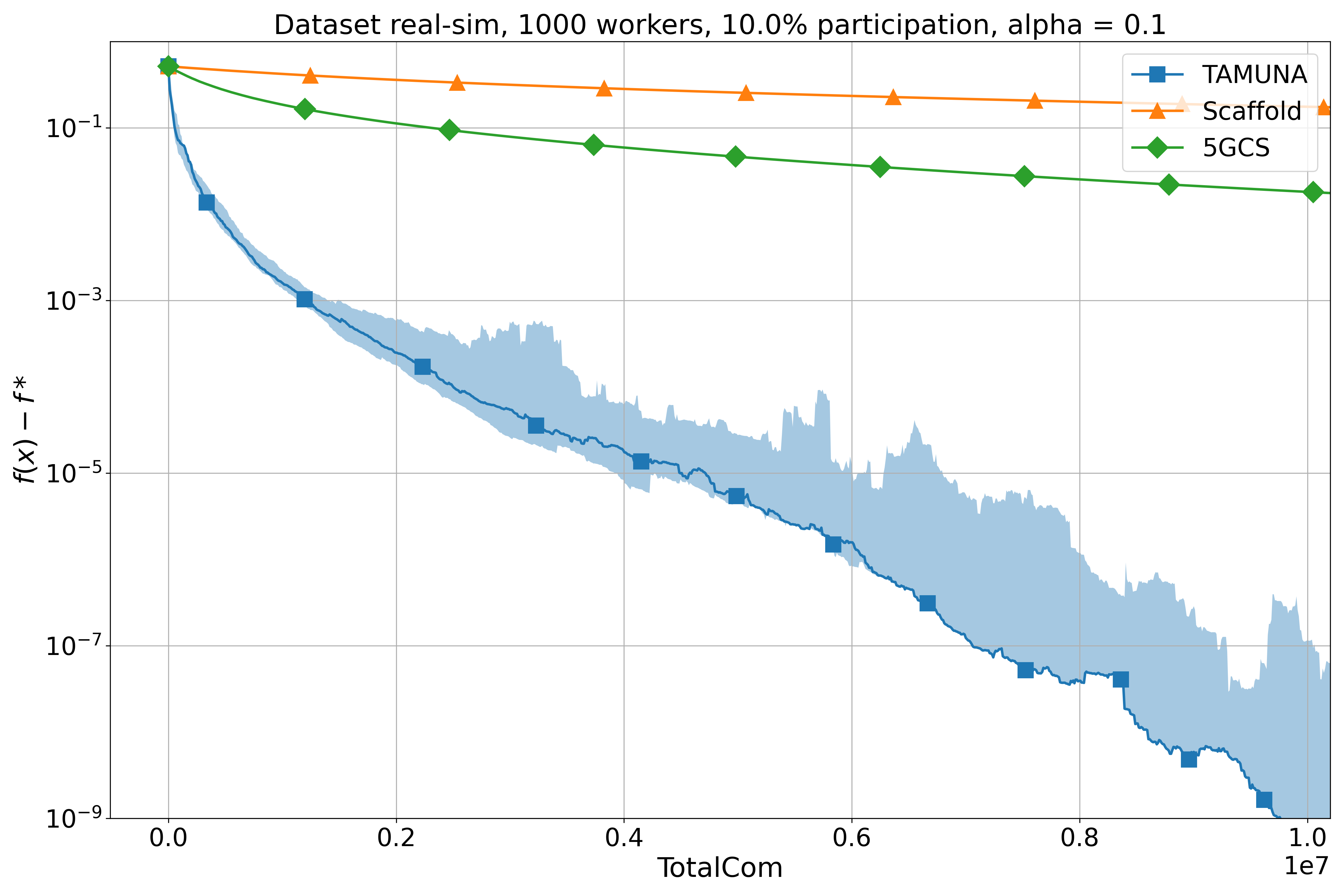}\\
		(c) real-sim, $n=1000$, $\alpha=0$, $c=0.1n$&(d) real-sim, $n=1000$, $\alpha=0.1$, $c=0.1n$
	\end{tabular}
	\caption{Logistic regression experiment in the case $d>n$. The dataset real-sim has $d=20,958$ features and $n=1000$, so $n \approx d/20$. The first row shows a comparison in the full participation regime, while the second row shows a comparison in the partial participation regime with 10\% of clients. On the left, $\myc=0$, while on the right, $\myc=0.1$.
	\label{fig:totalcostrealsim}}
\end{figure*}

\section{Experiments}\label{secexpe}

We empirically validate our theoretical findings by evaluating \algno on a distributed logistic regression problem. The global objective function is 
\begin{equation}
	f(x) = \frac{1}{M}\sum_{m=1}^M \left(\log\!\left(1 + \exp\!\left(-b_m a_m^{\top}x\right)\right) + \frac{\mu}{2}\| x \|^2\right),\label{eq13}
\end{equation}
where $(a_m, b_m) \in \mathbb{R}^d \times \{-1, 1\}$ represent the feature vectors and corresponding labels, respectively, and $M$ denotes the total number of samples. The global dataset is partitioned uniformly across $n$ clients to construct the local functions $f_i$ (discarding any minor remainder to ensure equal data distribution). To rigorously benchmark the empirical convergence, 
we tune the $\ell_2$-regularization parameter $\mu$ such that the global condition number is $\kappa = 10^4$.

We simulate $n=1000$ clients and investigate two distinct regimes using standard datasets from the LIBSVM library \citep{CC01a}: A high-dimensional regime ($d > n$) using the \texttt{real-sim} dataset ($d = 20,958$), and a low-dimensional regime ($n > d$) using the \texttt{w8a} dataset ($d = 300$). For each dataset, we evaluate the algorithmic performance under two specific communication asymmetry scenarios: a strictly UpCom-bottlenecked regime ($\myc=0$) and a more realistic, moderately asymmetric regime ($\myc=0.1$), where $\myc$ dictates the downlink penalty in the TotalCom model defined in \eqref{eqtotcom}.

We track the functional suboptimality gap $f(x)-f(x^\star)$ as a function of TotalCom complexity (i.e., the total number of real values transmitted per client). Here, $x$ denotes the current global model aggregated at the server, which corresponds to $\bar{x}^{(r)}$ for \algno. Because $f$ is $\LL$-smooth, the functional suboptimality is strictly bounded by $\frac{\LL}{2}\sqnorm{x-x^\star}$. Consequently, this objective error converges linearly at the same theoretical rate as the Lyapunov function $\Psi$ analyzed in Theorem~\ref{theo1}, making it the most direct and natural metric for comparing algorithmic efficiency.

We benchmark \algno against three established variance-reduced baseline algorithms. To specifically test robustness to client dropout, we compare against \algn{Scaffold} (the standard unaccelerated LT algorithm with PP) and \algn{5GCS} (the only other known accelerated algorithm supporting LT and PP). We evaluate these across two participation scenarios: full participation ($c=n$) and a realistic dropout regime with strictly 10\% active client participation ($c=0.1n$). Finally, in the full participation scenario, we additionally include \algn{Scaffnew} to clearly demonstrate how \algno's novel architecture and compression mechanisms improve upon the prior state of the art.

To strictly satisfy the theoretical conditions guaranteeing linear convergence, we set the stepsizes $\gamma$ and $\eta$ for \algno as:
\begin{equation*}
	\gamma = \frac{2}{\LL+\mu}, \quad \eta=  p\frac{n(s-1)}{s(n-1)},
\end{equation*}
while the remaining parameters $s$ (sparsity index) and $p$ (communication probability) are fine-tuned to optimize the overall communication complexity. Empirically, we observed the best performance trade-offs by setting $s=40$ and $p=0.01$. Because these values strictly adhere to the bounds established in Theorem~\ref{theo1}, the linear convergence of \algno is theoretically guaranteed. To ensure a rigorous and fair comparison, we adopt the same values of $\gamma$ and $p$ for \algn{Scaffnew}. For \algn{Scaffold}, we fix the number of local steps to $p^{-1}$, perfectly matching the expected number of local steps per round in both \algno and \algn{Scaffnew} (we observed that alternative values for \algn{Scaffold} yielded marginal changes). We also tune \algn{Scaffold}'s stepsize $\gamma$ to the highest empirical value that safely avoids divergence. For \algn{5GCS}, we grid-searched the parameters 
to obtain its best performance. 
Across all evaluated algorithms, the primal models, as well as the dual control variates in \algno, \algn{Scaffnew}, and \algn{Scaffold}, are globally initialized as zero vectors.

The convergence trajectories are plotted in Figures \ref{fig:totalcostw8a} and \ref{fig:totalcostrealsim}. To account for algorithmic variance, we execute multiple independent runs across different random seeds.
The shaded regions in the plots denote the min-max convergence gap across these independent seeds, while the solid, marked lines trace the progress of the first run.

\textbf{Discussion of Results:} As clearly illustrated in the figures, our proposed \algno systematically outperforms all baseline methods. Under full participation, \algn{Scaffnew} beats both \algn{Scaffold} and \algn{5GCS}, validating the theoretical superiority of its LT mechanism. While \algno embeds and benefits from this exact same LT mechanism, it outperforms \algn{Scaffnew} due to its secondary acceleration engine, compression. As predicted by our theory, the performance gap between \algno and \algn{Scaffnew} is most pronounced when uplink communication is the strict bottleneck ($\myc=0$) and naturally narrows when $\myc$ increases.

\section{Conclusion}

We introduced \algno, the first communication-efficient distributed optimization algorithm to successfully intertwine partial participation with the combined acceleration mechanisms of local training and compression. By decoupling the primal and dual updates, \algno provably achieves a doubly-accelerated communication complexity, for any level of client participation.
Furthermore, we establish these guarantees not merely for uplink transmissions, but within a comprehensive, asymmetric model of total communication (TotalCom). As corroborated by our experimental evaluations, these theoretical advantages translate directly to practice, allowing \algno to reach target accuracies with strictly less communication overhead than existing methods. 

Important avenues for future work include generalizing our specific permutation-based compression mechanism to accommodate a broader class of generic compressors, including quantization  \cite{hor22d,con25}. Additionally, implementing internal variance reduction for local stochastic gradients, analogous to the extensions developed for \algn{Scaffnew} \cite{mal22}, is a promising direction.

\section*{Acknowledgement}

This work was supported by funding from King Abdullah University of Science and Technology (KAUST):\\
i) KAUST Baseline Research Scheme,\\
ii) Center of Excellence for Generative AI (award no.\ 5940),\\
iii) Competitive Research Grant (CRG) Program (award no.\ 6460),\\
iv) SDAIA-KAUST Center of Excellence in Data Science and Artificial Intelligence (SDAIA-KAUST AI).

\bibliography{IEEEabrv,biblio}

\clearpage
\appendix

\noindent{\huge \textbf{Appendix}}

\section{Proof of Theorem~\ref{theo1}}\label{secalg2}

We first prove convergence of Algorithm~\ref{alg1}, which is a single-loop version of \algno; that is, there is a unique loop over the iterations and there is one local step per iteration. 
 In Section~\ref{sectam}, we show that this yields a proof of Theorem~\ref{theo1} for \algno. We can note that in case of full participation ($c=n$, $\Omega^t\equiv [n]$), Algorithm~\ref{alg1} reverts to \algn{CompressedScaffnew} \citep{con26cs}.

To simplify the analysis of Algorithm~\ref{alg1}, we introduce vector notations: the problem \eqref{eqpro1} can be written as
\begin{equation}
\mathrm{find}\ \mathbf{x}^\star =\argmin_{\mathbf{x}\in\mathcal{X}}\  \mathbf{f}(\mathbf{x})\quad\mbox{s.t.}\quad W\mathbf{x}=0,\label{eqpro2}
\end{equation}
where $\mathcal{X}\eqdef\mathbb{R}^{d\times n}$, an element 
$\mathbf{x}=(x_i)_{i=1}^n \in \mathcal{X}$ is a collection of vectors $x_i\in \mathbb{R}^d$, $\mathbf{f}:\mathbf{x}\in \mathcal{X}\mapsto \sum_{i=1}^n f_i(x_i)$ is $\LL$-smooth and $\mu$-strongly convex, the linear operator $W:\mathcal{X}\rightarrow \mathcal{X}$ maps $\mathbf{x}=(x_i)_{i=1}^n $ to $(x_i-\frac{1}{n}\sum_{j=1}^n x_j)_{i=1}^n$. 
The constraint $W\mathbf{x}=0$ means that $\mathbf{x}$ minus its average is zero; that is, $\mathbf{x}$ has identical components $x_1 = \cdots = x_n$. Thus, \eqref{eqpro2} is indeed equivalent to \eqref{eqpro1}. We have $W=W^*=W^2$.

\begin{algorithm}[t]
	\caption{}
	\label{alg1}
	\begin{algorithmic}[1]
		\STATE \textbf{input:}  stepsizes $\gamma>0$, $\chi >0$; probability $p \in (0,1]$; number of participating clients $c\in \{2,\ldots,n\}$; compression index $s\in \{2,\ldots,c\}$;
initial estimates $x_1^0,\ldots,x_n^0 \in \mathbb{R}^d$ and $h_1^0, \ldots, h_n^0 \in \mathbb{R}^d$ such that $\sum_{i=1}^n h_i^0=0$, 
		sequence of independent coin flips $\theta^0,\theta^1,\ldots$  
		with $\mathrm{Prob}(\theta^t=1)= p$, 
		and for every $t$ with $\theta^t=1$, a subset $\Omega^t\subset [n]$ of size $c$ chosen uniformly at random and 
		a random binary mask $\mathbf{q}^t=(q_i^t)_{i\in\Omega^t} \in \mathbb{R}^{d \times c}$.  
		The compressed vector $\mathcal{C}^t_i(v)$ is $v$ multiplied elementwise by $q_i^t$.
		\FOR{$t=0,1,\ldots$}
		\FOR{$i=1,\ldots,n$, at clients in parallel,}
\STATE $\hat{x}_i^t\eqdef x_i^t -\gamma g_i^t+ \gamma h_i^t$, where $g_i^t$ is an unbiased stochastic estimate of $\nabla f_i(x_i^t)$ of variance $\sigma_i^2$
\IF{$\theta^t=1$}
\IF{$i\in\Omega^t$}
\STATE send $\hat{x}_i^t$ to the server, which aggregates $\bar{x}^t\eqdef \frac{1}{s}\sum_{j\in\Omega^t}  \mathcal{C}_j^t(\hat{x}_{j}^t)$ and broadcasts it to all clients%
\STATE $h_i^{t+1}\eqdef h_i^t + \frac{p\chi}{\gamma}\big(\mathcal{C}_i^t(\bar{x}^{t})-\mathcal{C}_i^t(\hat{x}_i^t)\big)$
\ELSE 
\STATE $h_i^{t+1}\eqdef h_i^t$
\ENDIF
\STATE $x_i^{t+1}\eqdef \bar{x}^{t}$
\ELSE
\STATE $x_i^{t+1}\eqdef \hat{x}_i^t$
\STATE $h_i^{t+1}\eqdef h_i^t$
\ENDIF
\ENDFOR
		\ENDFOR
	\end{algorithmic}
\end{algorithm}

\begin{algorithm}[t]
	\caption{\label{alg2}}
		\begin{algorithmic}
			\STATE  \textbf{input:} stepsizes $\gamma>0$, $\chi>0$; probability $p \in (0,1]$, parameter $\omega\geq 0$; number of participating clients $c\in \{2,\ldots,n\}$; compression index $s\in \{2,\ldots,c\}$;
			initial estimates $\mathbf{x}^0\in\mathcal{X}$ and $\mathbf{h}^0\in\mathcal{X}$ such that $\sum_{i=1}^n h_i^0=0$;
		sequence of independent coin flips $\theta^0,\theta^1,\ldots$  with $\mathrm{Prob}(\theta^t=1)= p$, 
		and for every $t$ with $\theta^t=1$, a subset $\Omega^t \subset [n]$ of size $c$ chosen uniformly at random and 
		a random binary mask $\mathbf{q}^t=(q_i^t)_{i\in\Omega^t} \in \mathbb{R}^{d \times c}$. 
		The compressed vector $\mathcal{C}^t_i(v)$ is $v$ multiplied elementwise by $q_i^t$.
			\FOR{$t=0, 1, \ldots$}
			\STATE $\mathbf{\hat{x}}^{t} \eqdef  \mathbf{x}^t -\gamma \mathbf{g}^t 
			+ \gamma \mathbf{h}^t$, where $\mathbf{g}^t=(g_i^t)_{i=1}^n\approx \nabla \mathbf{f}(\mathbf{x}^t)$
			\IF{$\theta^t=1$}
			\STATE $\mathbf{\bar{x}}^{t} \eqdef (\bar{x}^t)_{i=1}^n$, where $\bar{x}^t\eqdef \frac{1}{s}\sum_{j\in\Omega^t} \mathcal{C}_j^t(\hat{x}_{j}^t )$
			\STATE $\mathbf{x}^{t+1}\eqdef \mathbf{\bar{x}}^{t}$
			\STATE $\mathbf{d}^t\eqdef   (d_i^t)_{i=1}^n$ with $d_i^t =  \left\{ \begin{array}{l}
			(1+\omega) \left(\mathcal{C}_i^t(\hat{x}_{i}^t) -\mathcal{C}_i^t(\bar{x}^t)\right)\ \mbox{if }  i\in \Omega^t,\\
			0 \ \mbox{otherwise}\end{array}\right.$
			\ELSE
			\STATE $\mathbf{x}^{t+1}\eqdef \mathbf{\hat{x}}^{t}$
			\STATE $\mathbf{d}^t\eqdef 0$
			\ENDIF
			\STATE $\mathbf{h}^{t+1}\eqdef \mathbf{h}^t -\frac{p\chi}{\gamma(1+\omega)}\mathbf{d}^t$
			\ENDFOR
		\end{algorithmic}
	\end{algorithm}

We also rewrite Algorithm~\ref{alg1} using vector notations as Algorithm~\ref{alg2}.
It converges linearly:

\begin{theorem}[fast linear convergence]\label{theo3}
In Algorithm~\ref{alg2},
suppose that $0<\gamma < \frac{2}{\LL}$, $0<\chi\leq \frac{n(s-1)}{s(n-1)}$, $\omega = \frac{n-1}{p (s-1)}  -1$.
For every $t\geq 0$, define the Lyapunov function
\begin{equation}
\Psi^{t}\eqdef  \frac{1}{\gamma}\sqnorm{\mathbf{x}^{t}-\mathbf{x}^\star}+ \frac{\gamma(1+\omega)}{p\chi}\sqnorm{\mathbf{h}^{t}-\mathbf{h}^\star},\label{eqlya1j}
\end{equation}
where $\mathbf{x}^\star$ is the unique solution to \eqref{eqpro2} and $\mathbf{h}^\star \eqdef \nabla \mathbf{f}(\mathbf{x}^\star)$. Then 
Algorithm~\ref{alg2}  
converges linearly:  for every $t\geq 0$, 
\begin{equation*}
\Exp{\Psi^{t}}\leq \mrho^t \Psi^0 + \frac{\gamma\sigma^2}{1-\mrho},
\end{equation*}
where 
\begin{equation}
\mrho\eqdef  \max\left((1-\gamma\mu)^2,(\gamma \LL-1)^2,1-p^2\chi\frac{s-1}{n-1}\right)<1.\label{eqrate2j}
\end{equation}
Also, if $\sigma=0$, $(\mathbf{x}^t)_{t \ge 0}$ and $(\mathbf{\hat{x}}^{t})_{t\ge 0}$ both converge  to $\mathbf{x}^\star$ and $(\mathbf{h}^t)_{t\ge 0}$ converges to $\mathbf{h}^\star$, almost surely.
\end{theorem}\medskip

\begin{proof}
We consider the variables of Algorithm~\ref{alg2}. 
For every $t\geq 0$, we denote by $\mathcal{F}_0^t$ the $\sigma$-algebra generated by the collection of $\mathcal{X}$-valued random variables $\mathbf{x}^0,\mathbf{h}^0,\ldots, \mathbf{x}^t,\mathbf{h}^t$, and by $\mathcal{F}^t$ the $\sigma$-algebra generated by these variables, as well as the stochastic gradients $\mathbf{g}^t$. 
$\mathbf{d}^t$ is 
a random variable; as proved in Section~\ref{secrand1}, it satisfies the 3 following properties, on which the convergence analysis of Algorithm~\ref{alg2} relies: 
for every $t\geq 0$,
\begin{enumerate}
\item $\Exp{\mathbf{d}^t\;|\;\mathcal{F}^t}= W \mathbf{\hat{x}}^{t}$.
\item 
$\displaystyle
\Exp{ \sqnorm{\mathbf{d}^t - W \mathbf{\hat{x}}^{t}}\;|\;\mathcal{F}^t} \leq \omega  \sqnorm{W \mathbf{\hat{x}}^{t}}.
$
\item $\mathbf{d}^t$ belongs to the range of $W$; that is, $\sum_{i=1}^n d^t_i = 0$.
\end{enumerate}

We suppose that $\sum_{i=1}^n h^0_i = 0$. Then, it follows from the third property of $\mathbf{d}^t$ that, for every $t\geq 0$, 
$\sum_{i=1}^n h^t_i = 0$; that is, $W \mathbf{h}^t=\mathbf{h}^t$.\medskip

For every $t\geq 0$, we define
$\mathbf{\hat{h}}^{t+1}\eqdef  \mathbf{h}^t-\frac{p\chi}{\gamma} W  \mathbf{\hat{x}}^{t}$, 
$\mathbf{w}^t\eqdef \mathbf{x}^t-\gamma \mathbf{g}^t$ and $\mathbf{w}^\star\eqdef \mathbf{x}^\star-\gamma \nabla \mathbf{f}(\mathbf{x}^\star)$. We also define $\mathbf{\bar{x}}^{t\sharp}\eqdef(\bar{x}^{t\sharp})_{i=1}^n$, with $\bar{x}^{t\sharp}\eqdef \frac{1}{n}\sum_{i=1}^n \hat{x}_i^t$; that is, $\bar{x}^{t\sharp}$ is the exact average of the $\hat{x}_i^t$, of which $\bar{x}^t$ is an unbiased random estimate. 
Let $t\geq 0$. We have 
\begin{align*}
\Exp{\sqnorm{\mathbf{x}^{t+1}-\mathbf{x}^\star}\;|\;\mathcal{F}^t}&=
p \Exp{\sqnorm{\mathbf{\bar{x}}^{t}-\mathbf{x}^\star}\;|\;\mathcal{F}^t,\theta^t=1}+(1-p)\sqnorm{\mathbf{\hat{x}}^{t}-\mathbf{x}^\star},
\end{align*}
Since $\Exp{\mathbf{\bar{x}}^{t}\;|\;\mathcal{F}^t,\theta^t=1} = \mathbf{\bar{x}}^{t\sharp}$, we have
\begin{align*}
\Exp{\sqnorm{\mathbf{\bar{x}}^{t}-\mathbf{x}^\star}\;|\;\mathcal{F}^t,\theta^t=1}&=\sqnorm{\mathbf{\bar{x}}^{t\sharp}-\mathbf{x}^\star}
 +\Exp{\sqnorm{\mathbf{\bar{x}}^{t}-\mathbf{\bar{x}}^{t\sharp}}\;|\;\mathcal{F}^t,\theta=1},
 \end{align*}
with 
\begin{equation*}
\sqnorm{\mathbf{\bar{x}}^{t\sharp}-\mathbf{x}^\star}=\sqnorm{\mathbf{\hat{x}}^{t}-\mathbf{x}^\star}-\sqnorm{W  \mathbf{\hat{x}}^{t}}.
\end{equation*}
To evaluate $\Exp{\sqnorm{\mathbf{\bar{x}}^{t}-\mathbf{x}^\star}\;|\;\mathcal{F}^t,\theta^t=1}$, where the expectation is taken over both the active subset $\Omega^t$ and the compression mask $\mathbf{q}^t$, we first observe that the expectation and the squared Euclidean norm are separable across the $d$ coordinates. Consequently, we can analyze the coordinates independently of one another, despite the structural dependencies that exist between the rows of $\mathbf{q}^t$.
Furthermore, for any given coordinate $k\in[d]$, the two-step process of uniformly sampling $s$ elements from the $c$ available elements $\hat{x}_{i,k}^t$ (where $i\in\Omega^t$ and $\Omega^t$ is itself chosen uniformly at random) is statistically equivalent to directly sampling $s$ elements uniformly at random from the entire population of $n$ clients. Thus, for each coordinate $k\in[d]$, we can define an equivalent random subset $\widetilde{\Omega}_k^t \subset [n]$ of size $s$, chosen uniformly at random, which identifies the locations of the ones in the $k$-th row of $\mathbf{q}^t$. This allows us to express the aggregated coordinate as:
\begin{equation*}
\bar{x}^t_k =  \frac{1}{s}\sum_{i\in \widetilde{\Omega}_k^t} \hat{x}_{i,k}^t.
\end{equation*}
Then, as proved in \citet[Proposition 1]{con22mu},
	\begin{equation*}
\Exp{\sqnorm{\mathbf{\bar{x}}^{t}-\mathbf{\bar{x}}^{t\sharp}}\;|\;\mathcal{F}^t,\theta^t=1}=n \sum_{k=1}^d \mathbb{E}_{\widetilde{\Omega}_k^t}\!\left[\left(\frac{1}{s}\sum_{i\in\widetilde{\Omega}_k^t}\hat{x}_{i,k}^{t} - \frac{1}{n}\sum_{j=1}^n\hat{x}_{j,k}^{t}\right)^2\;|\;\mathcal{F}^t\right]=\nu \sqnorm{W \mathbf{\hat{x}}^{t}},
	\end{equation*}
	where
		\begin{equation}
	\nu\eqdef\frac{n-s}{s(n-1)} \in \left[0,\frac{1}{2}\right).\label{eqnu}
	\end{equation}
Moreover,
\begin{align*}
\sqnorm{ \mathbf{\hat{x}}^{t}-\mathbf{x}^\star}
&= \sqnorm{\mathbf{w}^t-\mathbf{w}^\star}+\gamma^2\sqnorm{ \mathbf{h}^t- \mathbf{h}^\star}
 +2\gamma\langle \mathbf{w}^t-\mathbf{w}^\star, \mathbf{h}^t- \mathbf{h}^\star\rangle
\\
&= \sqnorm{\mathbf{w}^t-\mathbf{w}^\star}-\gamma^2\sqnorm{ \mathbf{h}^t- \mathbf{h}^\star}
 +2\gamma\langle\mathbf{\hat{x}}^{t}-\mathbf{x}^\star, \mathbf{h}^t- \mathbf{h}^\star\rangle\\
 &= \sqnorm{\mathbf{w}^t-\mathbf{w}^\star}-\gamma^2\sqnorm{ \mathbf{h}^t- \mathbf{h}^\star}
 +2\gamma\langle\mathbf{\hat{x}}^{t}-\mathbf{x}^\star, \mathbf{\hat{h}}^{t+1}- \mathbf{h}^\star\rangle
  -2\gamma\langle\mathbf{\hat{x}}^{t}-\mathbf{x}^\star, \mathbf{\hat{h}}^{t+1}- \mathbf{h}^t\rangle\\
   &= \sqnorm{\mathbf{w}^t-\mathbf{w}^\star}-\gamma^2\sqnorm{ \mathbf{h}^t- \mathbf{h}^\star}
 +2\gamma\langle\mathbf{\hat{x}}^{t}-\mathbf{x}^\star, \mathbf{\hat{h}}^{t+1}- \mathbf{h}^\star\rangle
  +2p\chi \langle\mathbf{\hat{x}}^{t}-\mathbf{x}^\star, W\mathbf{\hat{x}}^t\rangle\\
    &= \sqnorm{\mathbf{w}^t-\mathbf{w}^\star}-\gamma^2\sqnorm{ \mathbf{h}^t- \mathbf{h}^\star}
 +2\gamma\langle\mathbf{\hat{x}}^{t}-\mathbf{x}^\star, \mathbf{\hat{h}}^{t+1}- \mathbf{h}^\star\rangle
  +2p\chi \sqnorm{W\mathbf{\hat{x}}^t}.
 \end{align*}
Hence,
\begin{align*}
\Exp{\sqnorm{\mathbf{x}^{t+1}-\mathbf{x}^\star}\;|\;\mathcal{F}^t}&=
p\sqnorm{\mathbf{\hat{x}}^{t}-\mathbf{x}^\star}-p\sqnorm{W  \mathbf{\hat{x}}^{t}}+p\nu\sqnorm{W  \mathbf{\hat{x}}^{t}}
+(1-p)\sqnorm{\mathbf{\hat{x}}^{t}-\mathbf{x}^\star}\\
&=\sqnorm{\mathbf{\hat{x}}^{t}-\mathbf{x}^\star}-p(1-\nu)\sqnorm{W  \mathbf{\hat{x}}^{t}}\\
&=\sqnorm{\mathbf{w}^t-\mathbf{w}^\star}-\gamma^2\sqnorm{ \mathbf{h}^t- \mathbf{h}^\star}
 +2\gamma\langle\mathbf{\hat{x}}^{t}-\mathbf{x}^\star, \mathbf{\hat{h}}^{t+1}- \mathbf{h}^\star\rangle\\
 &\quad
  +\big(2p\chi-p(1-\nu)\big) \sqnorm{W\mathbf{\hat{x}}^t}.
\end{align*}

On the other hand, using the 3 properties of $\mathbf{d}^t$  stated above, we have
\begin{align*}
\Exp{\sqnorm{\mathbf{h}^{t+1}- \mathbf{h}^\star}\;|\;\mathcal{F}^t}&\leq \sqnorm{\mathbf{h}^t- \mathbf{h}^\star-\frac{p\chi}{\gamma(1+\omega)}W  \mathbf{\hat{x}}^{t}}
+\frac{\omega p^2\chi^2}{\gamma^2(1+\omega)^2}\sqnorm{W  \mathbf{\hat{x}}^{t}}\notag\\
&= \sqnorm{\mathbf{h}^t- \mathbf{h}^\star+\frac{1}{1+\omega}\big(\mathbf{\hat{h}}^{t+1} - \mathbf{h}^t\big)}
+\frac{\omega}{(1+\omega)^2}\sqnorm{\mathbf{\hat{h}}^{t+1} - \mathbf{h}^t }\notag\\
&=\sqnorm{\frac{\omega}{1+\omega}\big(\mathbf{h}^t- \mathbf{h}^\star\big)+\frac{1}{1+\omega}\big(\mathbf{\hat{h}}^{t+1} - \mathbf{h}^\star\big)}
+\frac{\omega}{(1+\omega)^2}\sqnorm{\mathbf{\hat{h}}^{t+1} - \mathbf{h}^t }\notag\\
&=\frac{\omega^2}{(1+\omega)^2}\sqnorm{\mathbf{h}^t- \mathbf{h}^\star}+\frac{1}{(1+\omega)^2}\sqnorm{\mathbf{\hat{h}}^{t+1}- \mathbf{h}^\star}\notag\\
&\quad+\frac{2\omega}{(1+\omega)^2}\langle \mathbf{h}^t- \mathbf{h}^\star,
\mathbf{\hat{h}}^{t+1}- \mathbf{h}^\star\rangle+\frac{\omega}{(1+\omega)^2}\sqnorm{\mathbf{\hat{h}}^{t+1} - \mathbf{h}^\star }\notag\\
&\quad+\frac{\omega}{(1+\omega)^2}\sqnorm{\mathbf{h}^t - \mathbf{h}^\star }-\frac{2\omega}{(1+\omega)^2}\langle \mathbf{h}^t- \mathbf{h}^\star,
\mathbf{\hat{h}}^{t+1}- \mathbf{h}^\star\rangle\notag\\
&=\frac{1}{1+\omega}\sqnorm{\mathbf{\hat{h}}^{t+1} - \mathbf{h}^\star }+\frac{\omega}{1+\omega}\sqnorm{\mathbf{h}^t - \mathbf{h}^\star }.
\end{align*}
Moreover,
\begin{align*}
\sqnorm{\mathbf{\hat{h}}^{t+1}- \mathbf{h}^\star}&=\sqnorm{( \mathbf{h}^t- \mathbf{h}^\star)+(\mathbf{\hat{h}}^{t+1}- \mathbf{h}^t)}\\
&=\sqnorm{ \mathbf{h}^t- \mathbf{h}^\star}+\sqnorm{\mathbf{\hat{h}}^{t+1}- \mathbf{h}^t}+2\langle  \mathbf{h}^t- \mathbf{h}^\star,\mathbf{\hat{h}}^{t+1}- \mathbf{h}^t\rangle\\
&=\sqnorm{ \mathbf{h}^t- \mathbf{h}^\star}+2 \langle \mathbf{\hat{h}}^{t+1}- \mathbf{h}^\star,\mathbf{\hat{h}}^{t+1}- \mathbf{h}^t\rangle - \sqnorm{\mathbf{\hat{h}}^{t+1}- \mathbf{h}^t}\\
&=\sqnorm{ \mathbf{h}^t- \mathbf{h}^\star} - \sqnorm{\mathbf{\hat{h}}^{t+1}- \mathbf{h}^t}-2\frac{p\chi}{\gamma} \langle \mathbf{\hat{h}}^{t+1}- \mathbf{h}^\star, W( \mathbf{\hat{x}}^{t}-\mathbf{x}^\star)\rangle\\
&=\sqnorm{ \mathbf{h}^t- \mathbf{h}^\star} - \frac{p^2\chi^2}{\gamma^2}\sqnorm{W\mathbf{\hat{x}}^t} -2\frac{p\chi}{\gamma} \langle W(\mathbf{\hat{h}}^{t+1}- \mathbf{h}^\star),  \mathbf{\hat{x}}^{t}-\mathbf{x}^\star\rangle\\
&=\sqnorm{ \mathbf{h}^t- \mathbf{h}^\star} - \frac{p^2\chi^2}{\gamma^2}\sqnorm{W\mathbf{\hat{x}}^t} -2\frac{p\chi}{\gamma} \langle \mathbf{\hat{h}}^{t+1}- \mathbf{h}^\star,  \mathbf{\hat{x}}^{t}-\mathbf{x}^\star\rangle.
\end{align*}
Hence, 
\begin{align}
\frac{1}{\gamma}\Exp{\sqnorm{\mathbf{x}^{t+1}-\mathbf{x}^\star}\;|\;\mathcal{F}^t}&+\frac{\gamma(1+\omega)}{p\chi}\Exp{\sqnorm{\mathbf{h}^{t+1}- \mathbf{h}^\star}\;|\;\mathcal{F}^t}\notag\\
&\leq  \frac{1}{\gamma}\sqnorm{\mathbf{w}^t-\mathbf{w}^\star}-\gamma\sqnorm{ \mathbf{h}^t- \mathbf{h}^\star}+\left(2\frac{p\chi}{\gamma}-\frac{p}{\gamma}(1-\nu)\right)\sqnorm{W\mathbf{\hat{x}}^t}\notag\\
&\quad+2 \langle
 \mathbf{\hat{x}}^{t}-\mathbf{x}^\star,\mathbf{\hat{h}}^{t+1}- \mathbf{h}^\star\rangle+\frac{\gamma}{p\chi}\sqnorm{ \mathbf{h}^t- \mathbf{h}^\star}\notag\\
&\quad- \frac{p\chi}{\gamma}\sqnorm{W\mathbf{\hat{x}}^t} -2 \langle \mathbf{\hat{h}}^{t+1}- \mathbf{h}^\star,  \mathbf{\hat{x}}^{t}-\mathbf{x}^\star\rangle +\frac{\gamma\omega}{p\chi}\sqnorm{\mathbf{h}^t - \mathbf{h}^\star }\notag\\
&=  \frac{1}{\gamma}\sqnorm{\mathbf{w}^t-\mathbf{w}^\star}+\left(\frac{\gamma(1+\omega)}{p\chi}-\gamma\right)\sqnorm{ \mathbf{h}^t- \mathbf{h}^\star}\notag\\
&\quad +\left(\frac{p\chi}{\gamma}-\frac{p(1-\nu)}{\gamma}\right)\sqnorm{W\mathbf{\hat{x}}^t}.\label{eq55}
\end{align}
Since we have supposed 
\begin{equation*}
0<\chi\leq 1-\nu =  \frac{n(s-1)}{s(n-1)} \in \left(\frac{1}{2},1\right],
\end{equation*}
we have
\begin{align*}
&\frac{1}{\gamma}\Exp{\sqnorm{\mathbf{x}^{t+1}-\mathbf{x}^\star}\;|\;\mathcal{F}^t}+\frac{\gamma(1+\omega)}{p\chi}\Exp{\sqnorm{\mathbf{h}^{t+1}- \mathbf{h}^\star}\;|\;\mathcal{F}^t}\\
&\qquad\qquad\leq  \frac{1}{\gamma}\sqnorm{\mathbf{w}^t-\mathbf{w}^\star}+\frac{\gamma(1+\omega)}{p\chi}\left(1-\frac{p\chi}{1+\omega}\right)\sqnorm{ \mathbf{h}^t- \mathbf{h}^\star}.
\end{align*}
Finally, 
\begin{equation*}
\Exp{\sqnorm{\mathbf{w}^t-\mathbf{w}^\star}\;|\;\mathcal{F}_0^t}\leq\sqnorm{(\mathrm{Id}-\gamma\nabla \mathbf{f})\mathbf{x}^t-(\mathrm{Id}-\gamma\nabla \mathbf{f})\mathbf{x}^\star} + \gamma^2\sigma^2,
\end{equation*}
and according to \citet[Lemma 1]{con22rp},
\begin{align*}
\sqnorm{(\mathrm{Id}-\gamma\nabla \mathbf{f})\mathbf{x}^t-(\mathrm{Id}-\gamma\nabla \mathbf{f})\mathbf{x}^\star} 
&\leq \max(1-\gamma\mu,\gamma \LL-1)^2 \sqnorm{\mathbf{x}^t-\mathbf{x}^\star}.
\end{align*}
Therefore,
\begin{align}
\Exp{\Psi^{t+1}\;|\;\mathcal{F}_0^t}&\leq \max\left((1-\gamma\mu)^2,(\gamma \LL-1)^2,1-\frac{p\chi}{1+\omega}\right)\Psi^t+\gamma\sigma^2\notag\\
&=\max\left((1-\gamma\mu)^2,(\gamma \LL-1)^2,1-p^2\chi\frac{s-1}{n-1}\right)\Psi^t+\gamma\sigma^2.\label{eqrec2b}
\end{align}
By applying the tower property of conditional expectation, we can recursively unroll \eqref{eqrec2b} to obtain the unconditional expectation of $\Psi^{t}$, which yields the stated linear convergence rate.

Furthermore, in the exact gradient regime ($\sigma=0$), the relation in \eqref{eqrec2b} implies that the sequence $(\Psi^t)_{t\in\mathbb{N}}$ forms a nonnegative supermartingale. Invoking standard supermartingale convergence theorems \citep[Proposition A.4.5]{ber15}, it follows that $\Psi^t \rightarrow 0$ almost surely. This consequently guarantees the almost sure convergence of both the primal variables $\mathbf{x}^t$ 
and the dual variables $\mathbf{h}^t$. 
Finally, by exploiting the Lipschitz continuity of $\nabla \mathbf{f}$, the squared distance $\|\mathbf{\hat{x}}^t-\mathbf{x}^\star\|^2$ can be upper bounded by a linear combination of $\|\mathbf{x}^t-\mathbf{x}^\star\|^2$ and $\| \mathbf{h}^t- \mathbf{h}^\star\|^2$. As a direct consequence, we deduce that $\Exp{\sqnorm{\mathbf{\hat{x}}^t-\mathbf{x}^\star}}$ also converges linearly to zero at the identical rate $\mrho$, and that $\mathbf{\hat{x}}^t \rightarrow \mathbf{x}^\star$ almost surely, thereby completing the proof.
\end{proof}

\subsection{The Random Variable $\mathbf{d}^t$}\label{secrand1}

We study  the random variable $\mathbf{d}^t$ used in Algorithm~\ref{alg2}. 
If $\theta^t = 0$, $\mathbf{d}^t= 0$. If, on the other hand, $\theta^t = 1$,  
 for every coordinate $k\in[d]$,
a subset $\widetilde{\Omega}_k^t \subset [n]$ of size $s$ is chosen uniformly at random. These sets $(\widetilde{\Omega}_k^t)_{k=1}^d$ are mutually dependent, but this does not matter for the derivations, since we can reason on the coordinates separately. 
Then, for every $k\in[d]$ and  $i\in[n]$,
\begin{equation*}
d^t_{i,k} \eqdef \left\{ \begin{array}{l}
a \left(\hat{x}_{i,k}^t - \frac{1}{s}\sum_{j\in \widetilde{\Omega}_k^t} \hat{x}_{j,k}^t\right)\ \mbox{if }  i\in \widetilde{\Omega}_k^t,\\
0 \ \mbox{otherwise},
\end{array}\right.
\end{equation*}
for some value $a>0$ to determine. We can check that $\sum_{i=1}^n d^t_i = 0$. We can also note that $\mathbf{d}^t$ depends only on $W\mathbf{\hat{x}}^{t}$ and not on $\mathbf{\hat{x}}^{t}$; in particular, if $\hat{x}_1^t=\cdots=\hat{x}_n^t$, $d_i^t=0$.
We have to set $a$ so that $\Exp{d_i^t}=\hat{x}_i^{t} - \frac{1}{n}\sum_{j=1}^n \hat{x}_j^{t}$, where the expectation is with respect to $\theta^t$ and the $\widetilde{\Omega}_k^t$  (all expectations in this section are conditional to $\mathbf{\hat{x}}^{t}$). So, let us calculate this expectation.

Let $k\in[d]$. For every $i\in [n]$, 
\begin{equation*}
\Exp{d^t_{i,k}}=p\frac{s}{n}\left(a\hat{x}_{i,k}^t - \frac{a}{s} \mathbb{E}_{\Omega: i\in\Omega}\!\left[  \sum_{j\in \Omega} \hat{x}_{j,k}^t\right]
\right),
\end{equation*}
 where $\mathbb{E}_{\Omega: i\in\Omega}$ denotes the expectation with respect to a subset $\Omega \subset [n]$ of size $s$ containing $i$ and chosen uniformly at random.
We have 
\begin{equation*}
\mathbb{E}_{\Omega: i\in\Omega}\!\left[  \sum_{j\in \Omega} \hat{x}_{j,k}^t\right] = \hat{x}_{i,k}^t + \frac{s-1}{n-1} \sum_{j\in [n]\backslash \{i\}} \hat{x}_{j,k}^t = \frac{n-s}{n-1}\hat{x}_{i,k}^t + \frac{s-1}{n-1} \sum_{j=1}^n \hat{x}_{j,k}^t.\label{eqgejkjg}
\end{equation*}
Hence, for every $i\in[n]$,
\begin{equation*}
\Exp{d^t_{i,k}}=p\frac{s}{n}\left(a - \frac{a}{s}\frac{n-s}{n-1}\right)\hat{x}_{i,k} - p\frac{s}{n}\frac{a}{s}\frac{s-1}{n-1} \sum_{j=1}^n \hat{x}_{j,k}.
\end{equation*}
Therefore, by setting 
\begin{equation*}
a\eqdef \frac{n-1}{p(s-1)},
\end{equation*}
we have, for every $i\in[n]$,
\begin{align*}
\Exp{d^t_{i,k}}&=p\frac{s}{n}\left(\frac{1}{p}\frac{n-1}{s-1} - \frac{1}{p}\frac{n-s}{s(s-1)}\right)\hat{x}_{i,k} - \frac{1}{n} \sum_{j=1}^n \hat{x}_{j,k}\\
&=\hat{x}_{i,k} - \frac{1}{n} \sum_{j=1}^n \hat{x}_{j,k},
\end{align*}
as desired.

Now, we want to find the value of $\omega$ such that 
\begin{equation}
\Exp{ \sqnorm{\mathbf{d}^t - W \mathbf{\hat{x}}^{t}}} \leq \omega  \sqnorm{W \mathbf{\hat{x}}^{t}}\label{hisihs}
\end{equation}
 or, equivalently,
\begin{equation*}
\Exp{\sum_{i=1}^n \sqnorm{d_i^t }} \leq (1+\omega) \sum_{i=1}^n \sqnorm{ \hat{x}_i^{t} - \frac{1}{n}\sum_{j=1}^n \hat{x}_j^{t}}.
\end{equation*}
We can reason on the coordinates separately, or all at once to ease the notations. We have
\begin{align*}
\Exp{\sum_{i=1}^n \sqnorm{d_i^t }} &= p\frac{s}{n}\sum_{i=1}^n \mathbb{E}_{\Omega: i\in\Omega} \sqnorm{a \hat{x}_i^{t}  -\frac{a}{s}\sum_{j\in\Omega} \hat{x}_j^{t} }.
\end{align*}
For every $i\in[n]$, 
\begin{align*}
 \mathbb{E}_{\Omega: i\in\Omega} \sqnorm{a \hat{x}_i^{t}  -\frac{a}{s}\sum_{j\in\Omega} \hat{x}_j^{t}}&= \mathbb{E}_{\Omega: i\in\Omega}\sqnorm{\left(a-\frac{a}{s}\right) \hat{x}_i^{t}  -\frac{a}{s}\sum_{j\in\Omega\backslash\{i\}} \hat{x}_j^{t} }\\
 &=\sqnorm{\left(a-\frac{a}{s}\right) \hat{x}_i^{t} }+\mathbb{E}_{\Omega: i\in\Omega}
 \sqnorm{\frac{a}{s}\sum_{j\in\Omega\backslash\{i\}} \hat{x}_j^{t} }\\
 &\quad -2 \left\langle \left(a-\frac{a}{s}\right) \hat{x}_i^{t}  ,
 \frac{a}{s} \mathbb{E}_{\Omega: i\in\Omega}\sum_{j\in\Omega\backslash\{i\}} \hat{x}_j^{t} \right\rangle.
\end{align*}
We have 
\begin{align*}
  \mathbb{E}_{\Omega: i\in\Omega}\sum_{j\in\Omega\backslash\{i\}} \hat{x}_j^{t} &=\frac{s-1}{n-1}\sum_{j\in[n]\backslash\{i\}} \hat{x}_j^{t} =\frac{s-1}{n-1}\left(\sum_{j=1}^n \hat{x}_j^{t} - \hat{x}_i^{t} \right)
\end{align*}
and
\begin{align*}
 \mathbb{E}_{\Omega: i\in\Omega}
 \sqnorm{\sum_{j\in\Omega\backslash\{i\}} \hat{x}_j^{t} }&=  \mathbb{E}_{\Omega: i\in\Omega}\sum_{j\in\Omega\backslash\{i\}} 
 \sqnorm{\hat{x}_j^{t}}+ \mathbb{E}_{\Omega: i\in\Omega} \sum_{j\in\Omega\backslash\{i\}}\sum_{j'\in\Omega\backslash\{i,j\}}
 \left\langle \hat{x}_j^{t},\hat{x}_{j'}^{t}\right\rangle\\
 &= \frac{s-1}{n-1}\sum_{j\in[n]\backslash\{i\}} \sqnorm{\hat{x}_j^{t}} + \frac{s-1}{n-1}\frac{s-2}{n-2}
 \sum_{j\in[n]\backslash\{i\}}\sum_{j'\in[n]\backslash\{i,j\}}\left\langle \hat{x}_j^{t},\hat{x}_{j'}^{t}\right\rangle\\
& = \frac{s-1}{n-1}\left(1-\frac{s-2}{n-2}\right)\sum_{j\in[n]\backslash\{i\}} \sqnorm{\hat{x}_j^{t}}+ \frac{s-1}{n-1}\frac{s-2}{n-2}
\sqnorm{ \sum_{j\in[n]\backslash\{i\}}\hat{x}_j^{t}}\\
& = \frac{s-1}{n-1}\frac{n-s}{n-2}\left(\sum_{j=1}^n \sqnorm{\hat{x}_j^{t}}-\sqnorm{\hat{x}_i^{t}}\right)
+ \frac{s-1}{n-1}\frac{s-2}{n-2}
\sqnorm{ \sum_{j=1}^n \hat{x}_j^{t}- \hat{x}_i^{t}} .
 \end{align*}
 Hence, 
 \begin{align*}
\Exp{\sum_{i=1}^n \sqnorm{d_i^t }} &= p\frac{s}{n}\sum_{i=1}^n \sqnorm{\left(a-\frac{a}{s}\right) \hat{x}_i^{t}   }+
p s\frac{a^2}{(s)^2} \frac{s-1}{n-1}\frac{n-s}{n-2}\sum_{j=1}^n \sqnorm{\hat{x}_j^{t}}\\
&\quad -
p\frac{s}{n}\frac{a^2}{(s)^2}  
\frac{s-1}{n-1}\frac{n-s}{n-2}\sum_{i=1}^n  \sqnorm{\hat{x}_i^{t}}+ p\frac{s}{n}\frac{a^2}{(s)^2}   \frac{s-1}{n-1}\frac{s-2}{n-2}\sum_{i=1}^n
\sqnorm{ \sum_{j=1}^n \hat{x}_j^{t}- \hat{x}_i^{t}} \\
&\quad -2 p\frac{s}{n} \frac{a}{s} \frac{s-1}{n-1} \left(a-\frac{a}{s}\right)\sum_{i=1}^n \left\langle  \hat{x}_i^{t} ,
 \sum_{j=1}^n \hat{x}_j^{t} -  \hat{x}_i^{t}  \right\rangle\\
 &= \frac{(n-1)^2}{p s n} \sum_{i=1}^n \sqnorm{ \hat{x}_i^{t}   }+
 \frac{(n-1)^2}{p s (s-1) n} \frac{n-s}{n-2}\sum_{i=1}^n \sqnorm{\hat{x}_i^{t}}\\
&\quad+ \frac{1}{p s } \frac{s-2}{s-1}  \frac{n-1}{n-2}
\sqnorm{ \sum_{i=1}^n \hat{x}_i^{t}} -2 \frac{1}{p s n} \frac{s-2}{s-1}  \frac{n-1}{n-2}
\sqnorm{ \sum_{i=1}^n \hat{x}_i^{t}} \\
&\quad+ \frac{1}{p s n} \frac{s-2}{s-1}  \frac{n-1}{n-2}\sum_{i=1}^n
\sqnorm{ \hat{x}_i^{t}}  +2 \frac{n-1}{p s n}   \sum_{i=1}^n \sqnorm{\hat{x}_i^{t}} 
 -2 \frac{n-1}{p s n}  \sqnorm{ \sum_{i=1}^n   \hat{x}_i^{t} }\\
  &= \frac{(n-1)(n+1)}{p s n} \sum_{i=1}^n \sqnorm{ \hat{x}_i^{t}   }+
 \frac{(n-1)^2}{p s (s-1) n} \frac{n-s}{n-2}\sum_{i=1}^n \sqnorm{\hat{x}_i^{t}}\\
&\quad- \frac{n-1}{p s n} \frac{s}{s-1}  
\sqnorm{ \sum_{i=1}^n \hat{x}_i^{t}} + \frac{1}{p s n} \frac{s-2}{s-1}  \frac{n-1}{n-2}\sum_{i=1}^n
\sqnorm{ \hat{x}_i^{t}} \\
&= \frac{(n^2-1)(s-1)(n-2)+(n-1)^2(n-s)+(s-2)(n-1)}{p s (s-1)n(n-2)} \sum_{i=1}^n \sqnorm{ \hat{x}_i^{t}   }\\
&\quad- \frac{n-1}{p  (s-1)n}  
\sqnorm{ \sum_{i=1}^n \hat{x}_i^{t}}  \\
&= \frac{n-1}{p  (s-1)} \sum_{i=1}^n \sqnorm{ \hat{x}_i^{t}   }- \frac{n-1}{p  (s-1)n}  
\sqnorm{ \sum_{i=1}^n \hat{x}_i^{t}} \\
&= \frac{n-1}{p  (s-1)} \sum_{i=1}^n \sqnorm{ \hat{x}_i^{t} - \frac{1}{n} \sum_{j=1}^n \hat{x}_j^{t}}.
\end{align*}
Therefore, \eqref{hisihs} holds with
\begin{equation*}
\omega =\frac{n-1}{p (s-1)}  -1,
\end{equation*}
and we have $a=1+\omega$.

\subsection{From Algorithm~\ref{alg1} to \algno}\label{sectam}

\algno formally restructures the single-loop dynamics of Algorithm~\ref{alg1} into a practical two-loop architecture, grouping each sequence of consecutive local steps and its subsequent communication step into a distinct round. The fundamental insight connecting the two algorithms lies in the treatment of idle clients: if a client $i\notin \Omega^t$ does not participate in a communication step at iteration $t$ (when $\theta^t=1$), its local variable $x_i$ is subsequently overwritten by the global aggregated model $\bar{x}^t$ (step 12 of Algorithm~\ref{alg1}). Consequently, any local computations performed by this idle client since its last active participation are entirely overwritten and mathematically redundant. However, in the single-loop formulation, during iterations where communication does not occur ($\theta^t=0$), the subset of clients $\Omega^{t'}$ destined to participate in the next communication step (at some future iteration $t'>t$) remains undetermined. \algno resolves this by introducing a two-loop structure where the active cohort $\Omega^{(r)}$ is explicitly sampled at the \emph{beginning} of round $r$. Consequently, only the active clients ($i\in\Omega^{(r)}$) download the current global model estimate (step 6 of \algno) to initialize their local steps. Clients excluded from the cohort ($i\notin\Omega^{(r)}$) remain strictly idle, executing zero computations and avoiding all uplink or downlink communication for the entirety of the round. 

Because a round essentially comprises a sequence of non-communicating iterations ($\theta^t=0$) terminated by a single communicating iteration ($\theta^t=1$), the total number of local steps per round exactly follows a geometric distribution and can be sampled upfront. Under this practical reordering, Algorithm~\ref{alg1} and \algno are strictly mathematically equivalent. To distinguish between the two temporal frames, \algno employs a double index $(r,\ell)$ for rounds and local steps, respectively.

To bridge Theorem~\ref{theo3} to our main result in Theorem~\ref{theo1}, we map the double indices back to the global iteration index $t$, allowing us to express the convergence rate with respect to the total number of local steps rather than the randomized number of rounds. 

Crucially, we aim to formulate our convergence bounds strictly in terms of the variables explicitly computed in the practical \algno, avoiding any reliance on the ``ghost'' variables maintained by idle clients in the theoretical Algorithm~\ref{alg1}. For this reason, Theorem~\ref{theo1} is expressed with respect to the aggregated global model $\bar{x}^t$, which exclusively tracks the actively communicating clients and represents the true state of the server at any communication milestone. 
We denote the practical Lyapunov function as $\overline{\Psi}$ in \eqref{eqlya1jf} to distinguish it from the theoretical $\Psi$ in \eqref{eqlya1j}. Building on this equivalence, the remainder of our analysis in Section~\ref{secalg2} proceeds using the structure of Algorithm~\ref{alg1} (and its vector form, Algorithm~\ref{alg2}) while maintaining identical definitions and notations.

Formally, let $t\geq 0$. If no communication occurs ($\theta^t=0$), we sample a cohort $\Omega^t\subset [n]$ of size $c$ uniformly at random, alongside a random binary mask $\mathbf{q}^t=(q_i^t)_{i\in\Omega^t} \in \mathbb{R}^{d \times c}$, allowing us to define the virtual aggregate 
 $\bar{x}^t\eqdef \frac{1}{s}\sum_{j\in\Omega^t} \mathcal{C}_j^t(\hat{x}_{j}^t)$ (In the context of Theorem~\ref{theo1}, $\Omega^t$ and $\mathbf{q}^t$ are the ones  used at the end of the round, a valid choice given their independence from past events). If communication does occur ($\theta^t=1$), the variables $\Omega^t$, $\mathbf{q}^t$ and $\bar{x}^t$ are explicitly instantiated by the algorithm.  In either case, our objective is to bound 
$\mathbb{E}\!\left[\sqnorm{\bar{x}^{t}-x^\star}\;|\;\mathcal{F}^t\right]$, where the expectation operates over the joint randomness of $\Omega^t$ and $\mathbf{q}^t$. 
Using the derivations already obtained, we have
\begin{align*}
n\Exp{\sqnorm{\bar{x}^{t}-x^\star}\;|\;\mathcal{F}^t}&=\sqnorm{\mathbf{\hat{x}}^{t}-\mathbf{x}^\star}-\sqnorm{W  \mathbf{\hat{x}}^{t}}+\nu \sqnorm{W \mathbf{\hat{x}}^{t}}\\
    &= \sqnorm{\mathbf{w}^t-\mathbf{w}^\star}-\gamma^2\sqnorm{ \mathbf{h}^t- \mathbf{h}^\star}
 +2\gamma\langle\mathbf{\hat{x}}^{t}-\mathbf{x}^\star, \mathbf{\hat{h}}^{t+1}- \mathbf{h}^\star\rangle\\
&\quad  +(2p\chi +\nu-1)\sqnorm{W\mathbf{\hat{x}}^t}\\
    &\leq \sqnorm{\mathbf{w}^t-\mathbf{w}^\star}-\gamma^2\sqnorm{ \mathbf{h}^t- \mathbf{h}^\star}
 +2\gamma\langle\mathbf{\hat{x}}^{t}-\mathbf{x}^\star, \mathbf{\hat{h}}^{t+1}- \mathbf{h}^\star\rangle\\
&\quad  +\big(2p\chi -p(1-\nu)\big)\sqnorm{W\mathbf{\hat{x}}^t}.
 \end{align*}
Hence,
\begin{align*}
&\frac{n}{\gamma}\Exp{\sqnorm{\bar{x}^{t}-x^\star}\;|\;\mathcal{F}^t}+\frac{\gamma(1+\omega)}{p\chi}\Exp{\sqnorm{\mathbf{h}^{t+1}- \mathbf{h}^\star}\;|\;\mathcal{F}^t}\\
&\qquad\qquad\leq  \frac{1}{\gamma}\sqnorm{\mathbf{w}^t-\mathbf{w}^\star}+\frac{\gamma(1+\omega)}{p\chi}\left(1-\frac{p\chi}{1+\omega}\right)\sqnorm{ \mathbf{h}^t- \mathbf{h}^\star}
\end{align*}
and
\begin{align*}
&\frac{n}{\gamma}\Exp{\sqnorm{\bar{x}^{t}-x^\star}\;|\;\mathcal{F}_0^t}+\frac{\gamma(1+\omega)}{p\chi}\Exp{\sqnorm{\mathbf{h}^{t+1}- \mathbf{h}^\star}\;|\;\mathcal{F}_0^t}\\
&\qquad\qquad\leq\max\left((1-\gamma\mu)^2,(\gamma \LL-1)^2,1-p^2\chi\frac{s-1}{n-1}\right)\Psi^t+\gamma\sigma^2.
\end{align*}
Using the tower rule,
\begin{equation*}
\frac{n}{\gamma}\Exp{\sqnorm{\bar{x}^{t}-x^\star}}+\frac{\gamma(1+\omega)}{p\chi}\Exp{\sqnorm{\mathbf{h}^{t+1}- \mathbf{h}^\star}}\leq \mrho^t \Psi^0+\frac{\gamma\sigma^2}{1-\mrho}.
\end{equation*}
Since in \algno, $x_1^0 = \cdots = x_n^0 = \bar{x}^0 = \bar{x}^{(0)}$, then $\overline{\Psi}^0 = \Psi^0$. 
This concludes the proof of Theorem~\ref{theo1}.

\section{Proof of Theorem~\ref{theo2}}

We suppose that the assumptions in Theorem~\ref{theo2} hold. $s$ is set as the maximum of three values. Let us consider these three cases.

1) Suppose that $s=2$. 
Since $2=s \geq \lfloor \myc c\rfloor $ and $2=s\geq \lfloor\frac{c}{d}\rfloor$, we have  $\myc \leq \frac{3}{c}$ and $1\leq \frac{3d}{c}$.
Hence,
\begin{align}
&\mathcal{O}\left(\sqrt{\frac{n\kappa}{s}}+\frac{n}{s}\right)\!\left(\frac{sd}{c}+1+\myc d\right)\notag\\
={}&\mathcal{O}\left(\sqrt{n\kappa}+n\right)\!\left(\frac{d}{c}+\frac{d}{c}+\frac{d}{c}\right)\notag\\
={}&\mathcal{O}\left(d\frac{\sqrt{n\kappa}}{c}+d\frac{n}{c}\right).\label{eqb1}
\end{align}
 
 2) Suppose that $s=\lfloor\frac{c}{d}\rfloor$. Then $\frac{sd}{c}\leq 1$. Since $s \geq \lfloor \myc c\rfloor $ and $\lfloor\frac{c}{d}\rfloor=s\geq 2$, we have $\myc c\leq s+1\leq \frac{c}{d}+1$ and $\frac{d}{c}\leq \frac{1}{2}$, so that $\myc d \leq 1 + \frac{d}{c}\leq 2$. Hence,
\begin{align*}
&\mathcal{O}\left(\sqrt{\frac{n\kappa}{s}}+\frac{n}{s}\right)\!\left(\frac{sd}{c}+1+\myc d\right)\notag\\
={}&\mathcal{O}\left(\sqrt{\frac{n\kappa}{s}}+\frac{n}{s}\right).\notag
\end{align*}
Since $2s \geq \frac{c}{d}$, we have $\frac{1}{s} \leq \frac{2d}{c}$ and
\begin{align}
&\mathcal{O}\left(\sqrt{\frac{n\kappa}{s}}+\frac{n}{s}\right)\!\left(\frac{sd}{c}+1+\myc d\right)\notag\\
={}&\mathcal{O}\left(\sqrt{d}\sqrt{\frac{n\kappa}{c}}+d\frac{n}{c}\right).\label{eqb2}
\end{align}

 3) Suppose that $s=\lfloor \myc c\rfloor$. This implies $\myc>0$. Then $s\leq \myc c$. Also, $2s\geq \myc c$ and $\frac{1}{s}\leq \frac{2}{\myc c}$. Since $s=\lfloor \myc c\rfloor \geq \lfloor\frac{c}{d}\rfloor$, we have $\myc c+1\geq \frac{c}{d}$ and $1\leq \myc d +\frac{d}{c}$. Since $s=\lfloor \myc c\rfloor \geq 2$, we have $\frac{1}{c}\leq \frac{\myc}{2}$ and $1\leq 2\myc d$. Hence,
 \begin{align}
&\mathcal{O}\left(\sqrt{\frac{n\kappa}{s}}+\frac{n}{s}\right)\!\left(\frac{sd}{c}+1+\myc d\right)\notag\\
={}&\mathcal{O}\left(\sqrt{\frac{n\kappa}{\myc c}}+\frac{n}{\myc c}\right)\!\left(\myc d+\myc d+\myc d\right)\notag\\
={}&\mathcal{O}\left(\sqrt{\myc}d\sqrt{\frac{n\kappa}{c}}+d\frac{n}{c}\right).\label{eqb3}
\end{align}

By adding up the three upper bounds \eqref{eqb1}, \eqref{eqb2}, \eqref{eqb3}, we obtain the upper bound in \eqref{eqbt}.

\section{Sublinear Convergence in the Convex Case}\label{secmco}

In this section, we relax the strong convexity assumption. We assume only that the individual objective functions $f_i$ are convex and $\LL$-smooth, and that at least one global minimizer $x^\star \in \mathbb{R}^d$ to \eqref{eqpro1} exists. For simplicity, our analysis is restricted to the deterministic setting with exact gradients ($\sigma=0$). Under these relaxed conditions, we establish a sublinear ergodic convergence rate.

\begin{theorem}[sublinear convergence]\label{theo4}
In Algorithm~\ref{alg1}
suppose that $\sigma=0$ and that the parameters satisfy
\begin{equation*}
0<\gamma < \frac{2}{\LL}\quad\mbox{and}\quad 0<\chi<  \frac{n(s-1)}{s(n-1)} \in \left(\frac{1}{2},1\right].
\end{equation*}
For every $i=1,\ldots,n$ and $T\geq 0$, define the ergodic average
\begin{equation}
\tilde{x}_i^T \eqdef \frac{1}{T+1} \sum_{t=0}^T x_i^t.\label{eqwkh}
\end{equation}
Then, the gradient norm evaluated at the ergodic average converges sublinearly:
\begin{equation*}
 \Exp{\sqnorm{\nabla f(\tilde{x}_i^T)}} = \mathcal{O}\left(\frac{1}{T}\right).
\end{equation*}
\end{theorem}

\begin{proof}
Let $x^\star \in \mathbb{R}^d$ be any solution to \eqref{eqpro1}. The first-order optimality condition yields $\nabla f(x^\star)=\frac{1}{n}\sum_{i=1}^n\nabla f_i(x^\star)=0$. While the minimizer $x^\star$ is not necessarily unique, the local gradients evaluated at the optimum, $h_i^\star \eqdef \nabla f_i(x^\star)$, are invariant across all global minima due to the convexity and smoothness of $f_i$.

Recall that the Bregman divergence associated with an $\LL$-smooth convex function $g$ evaluated at points $x,x' \in\mathbb{R}^d$ is defined as $D_{g}(x,x')\coloneqq g(x)-g(x')-\langle \nabla g(x'),x-x'\rangle \geq 0$. Standard properties of smooth convexity guarantee that $D_{g}(x,x')\geq \frac{1}{2L} \| \nabla g(x)-\nabla g(x')\|^2$. Furthermore, for any $x\in\mathbb{R}^d$ and $i \in [n]$, the quantity $D_{f_i}(x,x^\star)$ remains identical regardless of which specific global minimizer $x^\star$ is chosen.

For every iteration $t\geq 0$, we define the Lyapunov function
\begin{equation}
\Psi^{t}\eqdef  \frac{1}{\gamma}\sum_{i=1}^n\sqnorm{x_i^{t}-x^\star}+ \frac{\gamma}{p^2 \chi}
\frac{n-1}{s-1}
\sum_{i=1}^n\sqnorm{h_i^{t}-h_i^\star},\label{eqlya1jfm}
\end{equation}
Recalling the single-step bound established in \eqref{eq55}, we have for any $t\geq 0$,
\begin{align*}
\Exp{\Psi^{t+1}\;|\;\mathcal{F}^t}&=
\frac{1}{\gamma}\sum_{i=1}^n\Exp{\sqnorm{x_i^{t+1}-x^\star}\;|\;\mathcal{F}^t}+\frac{\gamma}{p^2 \chi}
\frac{n-1}{s-1}\sum_{i=1}^n\Exp{\sqnorm{h_i^{t+1}- h_i^\star}\;|\;\mathcal{F}^t}\\
&\leq    \frac{1}{\gamma}\sum_{i=1}^n\sqnorm{\big(x_i^t-\gamma\nabla f_i(x_i^t)\big)-\big(x^\star-\gamma\nabla f_i(x^\star)\big)}\\
&\quad+\left(\frac{\gamma}{p^2 \chi}
\frac{n-1}{s-1}-\gamma\right)\sum_{i=1}^n\sqnorm{ h_i^t-h_i^\star}+\frac{p}{\gamma}(\chi-1+\nu)\sum_{i=1}^n\sqnorm{\hat{x}_i^t-\frac{1}{n}\sum_{j=1}^n \hat{x}_j^t},
\end{align*}
with
\begin{align*}
\sqnorm{\big(x_i^t-\gamma\nabla f_i(x_i^t)\big)-\big(x^\star-\gamma\nabla f_i(x^\star)\big)}&= \sqnorm{x_i^t-x^\star}  -2\gamma\langle \nabla f_i(x_i^t)-\nabla f_i(x^\star),x_i^t-x^\star\rangle \\
&\quad+ \gamma^2\sqnorm{\nabla f_i(x_i^t)-\nabla f_i(x^\star)}\\
&\leq \sqnorm{x_i^t-x^\star}-  (2\gamma-\gamma^2 \LL)\langle \nabla f_i(x_i^t)-\nabla f_i(x^\star),x_i^t-x^\star\rangle,
\end{align*}
where the second inequality follows from the cocoercivity of the gradient. Additionally, convexity ensures that 
$D_{f_i}(x,x')\leq \langle \nabla f_i(x)-\nabla f_i(x'),x-x'\rangle$  for every $x,x'$. 
Substituting these bounds gives
\begin{align*}
\Exp{\Psi^{t+1}\;|\;\mathcal{F}^t}&\leq    \Psi^t  -(2-\gamma \LL) \sum_{i=1}^n D_{f_i}(x_i^t,x^\star)\\
&\quad -\gamma \sum_{i=1}^n\sqnorm{ h_i^t-h_i^\star}+\frac{p}{\gamma}(\chi-1+\nu)\sum_{i=1}^n\sqnorm{\hat{x}_i^t-\frac{1}{n}\sum_{j=1}^n \hat{x}_j^t}.
\end{align*}
By telescoping this inequality from $t=0$ to $T$ and invoking the tower property of conditional expectation, we establish the summability of the three nonpositive terms: for every $T\geq 0$, we have
\begin{align}
&(2-\gamma \LL) \sum_{i=1}^n \sum_{t=0}^T \Exp{D_{f_i}(x_i^t,x^\star)} \leq \Psi^0-\Exp{\Psi^{T+1}}\leq \Psi^0,\label{eqjghhrgm}\\
&\gamma \sum_{i=1}^n \sum_{t=0}^T \Exp{\sqnorm{ h_i^t-h_i^\star}} \leq \Psi^0-\Exp{\Psi^{T+1}}\leq \Psi^0,\label{eqjghhrgm2}\\
&\frac{p}{\gamma}(1-\nu-\chi)\sum_{i=1}^n \sum_{t=0}^T \Exp{\sqnorm{\hat{x}_i^t-\frac{1}{n}\sum_{j=1}^n \hat{x}_j^t}} \leq \Psi^0-\Exp{\Psi^{T+1}}\leq \Psi^0.\label{eqjghhrgm3}
\end{align}
To derive the formal $\mathcal{O}(1/T)$ rate, we pass to ergodic averages and exploit the convexity of both the squared Euclidean norm and the Bregman divergence. We use a tilde to denote the time-averaged quantities. Specifically, for every $i \in [n]$ and $T\geq 0$, we define
\begin{equation*}
\tilde{x}_i^T \eqdef \frac{1}{T+1} \sum_{t=0}^T x_i^t \quad \text{and} \quad \tilde{x}^T \eqdef \frac{1}{n} \sum_{i=1}^n \tilde{x}_i^T.
\end{equation*}
Because the Bregman divergence is convex in its first argument, we have for every $T\geq 0$,
\begin{equation*}
\sum_{i=1}^n D_{f_i}(\tilde{x}_i^T,x^\star) \leq \sum_{i=1}^n \frac{1}{T+1} \sum_{t=0}^T D_{f_i}(x_i^t,x^\star).
\end{equation*}
Combining this with \eqref{eqjghhrgm} yields
\begin{equation}
(2-\gamma \LL)\sum_{i=1}^n \Exp{D_{f_i}(\tilde{x}_i^T,x^\star)} \leq \frac{\Psi^0}{T+1}.\label{eqsub1}
\end{equation}
Similarly, defining the dual ergodic average $\tilde{h}_i^T \eqdef \frac{1}{T+1} \sum_{t=0}^T h_i^t$, we obtain, for every $i\in[n]$ and $T\geq 0$,
\begin{equation*}
\sum_{i=1}^n\sqnorm{ \tilde{h}_i^T-h_i^\star} \leq \sum_{i=1}^n\frac{1}{T+1} \sum_{t=0}^T \sqnorm{ h_i^t-h_i^\star}.
\end{equation*}
Combining this inequality with \eqref{eqjghhrgm2} yields,  for every $T\geq 0$,
\begin{equation}
\gamma\sum_{i=1}^n \Exp{\sqnorm{ \tilde{h}_i^T-h_i^\star}} \leq \frac{\Psi^0}{T+1}.\label{eqsub2}
\end{equation}
Finally, for every $i=1,\ldots,n$ and $T\geq 0$, we define
\begin{equation*}
\tilde{\hat{x}}_i^T \eqdef \frac{1}{T+1} \sum_{t=0}^T \hat{x}_i^t\quad \text{and} \quad \tilde{\hat{x}}^T \eqdef \frac{1}{n} \sum_{i=1}^n \tilde{\hat{x}}_i^T.
\end{equation*}
Then, we have, for every $T\geq 0$, 
\begin{equation*}
\sum_{i=1}^n\sqnorm{ \tilde{\hat{x}}_i^T-\tilde{\hat{x}}^T} \leq \sum_{i=1}^n\frac{1}{T+1} \sum_{t=0}^T \sqnorm{ \hat{x}_i^t-\frac{1}{n}\sum_{j=1}^n \hat{x}_j^t}.
\end{equation*}
Combining this inequality with \eqref{eqjghhrgm3} yields,  for every $T\geq 0$,
\begin{equation}
\frac{p}{\gamma}(1-\nu-\chi)\sum_{i=1}^n \Exp{\sqnorm{ \tilde{\hat{x}}_i^T-\tilde{\hat{x}}^T}} \leq \frac{\Psi^0}{T+1}.\label{eqsub3}
\end{equation}
Next, we have, for every $i=1,\ldots,n$ and $T\geq 0$,
\begin{align}
\sqnorm{\nabla f(\tilde{x}_i^T)}&\leq 2\sqnorm{\nabla f(\tilde{x}_i^T)-\nabla f(\tilde{x}^T)}+2\sqnorm{\nabla f(\tilde{x}^T)}\notag\\
&\leq 2\LL^2\sqnorm{\tilde{x}_i^T-\tilde{x}^T}+2\sqnorm{\nabla f(\tilde{x}^T)}.\label{eqco1}
\end{align}
Moreover, since $\nabla f(x^\star) = 0$, we have, for every $T\geq 0$,
\begin{align}
\sqnorm{\nabla f(\tilde{x}^T)}&=\sqnorm{\nabla f(\tilde{x}^T)-\nabla f(x^\star)}\notag\\
&\leq \frac{1}{n}\sum_{i=1}^n \sqnorm{\nabla f_i(\tilde{x}^T)-\nabla f_i(x^\star)}\notag\\
&\leq \frac{2}{n}\sum_{i=1}^n \sqnorm{\nabla f_i(\tilde{x}^T)-\nabla f_i(\tilde{x}_i^T)}+\frac{2}{n}\sum_{i=1}^n \sqnorm{\nabla f_i(\tilde{x}_i^T)-\nabla f_i(x^\star)}\notag\\
&\leq \frac{2\LL^2}{n}\sum_{i=1}^n \sqnorm{\tilde{x}_i^T-\tilde{x}^T}+\frac{4\LL}{n}\sum_{i=1}^n D_{f_i}(\tilde{x}_i^T,x^\star).\label{eqco2}
\end{align}
It remains to bound the consensus deviation $\sqnorm{\tilde{x}_i^T-\tilde{x}^T}$: for every $T\geq 0$,
\begin{align}
\sum_{i=1}^n\sqnorm{\tilde{x}_i^T-\tilde{x}^T}&\leq 2\sum_{i=1}^n\sqnorm{(\tilde{x}_i^T-\tilde{x}^T)-(\tilde{\hat{x}}_i^T-\tilde{\hat{x}}^T)}+2\sum_{i=1}^n\sqnorm{\tilde{\hat{x}}_i^T-\tilde{\hat{x}}^T}\notag\\
&\leq 2\sum_{i=1}^n\sqnorm{\tilde{x}_i^T-\tilde{\hat{x}}_i^T}+2\sum_{i=1}^n\sqnorm{\tilde{\hat{x}}_i^T-\tilde{\hat{x}}^T}.\label{eqco3}
\end{align}
For every $i\in[n]$ and $t\geq 0$,
\begin{equation*}
\hat{x}_i^t = x_i^t - \gamma \big(\nabla f_i(x_i^t) -h_i^t\big).
\end{equation*}
This implies
\begin{equation*}
\tilde{x}_i^T - \tilde{\hat{x}}_i^T =  \gamma \frac{1}{T+1}\sum_{t=0}^T\nabla f_i(x_i^t) -\gamma \tilde{h}_i^T,
\end{equation*}
and
\begin{align}
\sqnorm{\tilde{x}_i^T-\tilde{\hat{x}}_i^T}&= \gamma^2\sqnorm{ \frac{1}{T+1}\sum_{t=0}^T\nabla f_i(x_i^t) -\tilde{h}_i^T}\notag\\
&\leq 2\gamma^2  \frac{1}{T+1}\sum_{t=0}^T \sqnorm{\nabla f_i(x_i^t) -\nabla f_i(x^\star)}+2\gamma^2 \sqnorm{\tilde{h}_i^T-h_i^\star}\notag\\
&\leq 4\LL\gamma^2 \frac{1}{T+1}\sum_{t=0}^T D_{f_i}(x_i^t,x^\star)+2\gamma^2 \sqnorm{\tilde{h}_i^T-h_i^\star}.\label{eqco4}
\end{align}
Combining \eqref{eqco1}, \eqref{eqco2},  \eqref{eqco3},  \eqref{eqco4}, we obtain, for every $T\geq 0$,
\begin{align*}
\sum_{i=1}^n \sqnorm{\nabla f(\tilde{x}_i^T)}&\leq 2\LL^2\sum_{i=1}^n\sqnorm{\tilde{x}_i^T-\tilde{x}^T}+2n\sqnorm{\nabla f(\tilde{x}^T)}\\
&\leq 2\LL^2\sum_{i=1}^n\sqnorm{\tilde{x}_i^T-\tilde{x}^T}+2\LL^2\sum_{i=1}^n \sqnorm{\tilde{x}_i^T-\tilde{x}^T}+4\LL\sum_{i=1}^n D_{f_i}(\tilde{x}_i^T,x^\star)\\
&= 4\LL^2\sum_{i=1}^n\sqnorm{\tilde{x}_i^T-\tilde{x}^T}+4\LL\sum_{i=1}^n D_{f_i}(\tilde{x}_i^T,x^\star)\\
&\leq 8\LL^2\sum_{i=1}^n\sqnorm{\tilde{x}_i^T-\tilde{\hat{x}}_i^T}+8\LL^2\sum_{i=1}^n\sqnorm{\tilde{\hat{x}}_i^T-\tilde{\hat{x}}^T}+4\LL\sum_{i=1}^n D_{f_i}(\tilde{x}_i^T,x^\star)\\
&\leq
32\LL^3\gamma^2 \frac{1}{T+1}\sum_{i=1}^n\sum_{t=0}^T D_{f_i}(x_i^t,x^\star)+16\LL^2\gamma^2 \sum_{i=1}^n\sqnorm{\tilde{h}_i^T-h_i^\star}\\
&\quad+8\LL^2\sum_{i=1}^n\sqnorm{\tilde{\hat{x}}_i^T-\tilde{\hat{x}}^T}+4\LL\sum_{i=1}^n D_{f_i}(\tilde{x}_i^T,x^\star).
\end{align*}
Taking the unconditional expectation and substituting the bounds established in \eqref{eqjghhrgm}, \eqref{eqsub2}, \eqref{eqsub3}, and \eqref{eqsub1}, we conclude:
\begin{align*}
\sum_{i=1}^n \Exp{\sqnorm{\nabla f(\tilde{x}_i^T)}}&\leq
32\LL^3\gamma^2 \frac{1}{T+1}\sum_{i=1}^n\sum_{t=0}^T \Exp{D_{f_i}(x_i^t,x^\star)}\\
&\quad+16\LL^2\gamma^2 \sum_{i=1}^n\Exp{\sqnorm{\tilde{h}_i^T-h_i^\star}}\\
&\quad+8\LL^2\sum_{i=1}^n\Exp{\sqnorm{\tilde{\hat{x}}_i^T-\tilde{\hat{x}}^T}}+4\LL\sum_{i=1}^n \Exp{D_{f_i}(\tilde{x}_i^T,x^\star)}.\\
&\leq
\frac{32\LL^3\gamma^2}{2-\gamma \LL} \frac{\Psi_0}{T+1}+16\LL^2\gamma \frac{\Psi_0}{T+1}+ \frac{8\LL^2\gamma}{p(1-\nu-\chi)} \frac{\Psi_0}{T+1} +\frac{4\LL}{2-\gamma \LL}\frac{\Psi_0}{T+1}\\
&=\left[\frac{32\LL^3\gamma^2+4\LL}{2-\gamma \LL}+16\LL^2\gamma + \frac{8\LL^2\gamma}{p(1-\nu-\chi)} \right]\frac{\Psi_0}{T+1}.
\end{align*}
\end{proof}

Consequently, by setting $\gamma = \Theta\left(\frac{p}{\LL}\sqrt{\frac{c}{n}}\right)$ and choosing $\chi$ such that $\delta\leq \chi \leq 1-\nu-\delta$ for some arbitrarily small $\delta>0$, and assuming initialization $h_i^0 = \nabla f_i(x^0)$ for all $i\in[n]$, we guarantee that for any $\epsilon>0$, $
 \sum_{i=1}^n \Exp{\sqnorm{\nabla f(\tilde{x}_i^T)}}\leq \epsilon
$
is successfully achieved after
\begin{equation*}
\mathcal{O}\left(\frac{\LL^2}{p}\sqrt{\frac{n}{c}}
\frac{\sqnorm{\mathbf{x}^{0}-\mathbf{x}^\star}}{\epsilon}\right)
\end{equation*} 
total local iterations, which corresponds to 
\begin{equation*}
\mathcal{O}\left(\LL^2\sqrt{\frac{n}{c}}
\frac{\sqnorm{\mathbf{x}^{0}-\mathbf{x}^\star}}{\epsilon}\right)
\end{equation*} 
communication rounds. 

Notably, in this general convex regime, LT does not yield theoretical acceleration: the communication complexity remains independent of $p$. CC, however, remains highly effective, as the algorithm transmits vastly fewer than $d$ real values per communication round.

Finally, while this theoretical analysis formally applies to the single-loop Algorithm~\ref{alg1} (where $\tilde{x}_i^T$ averages all local states, including the ``ghost'' variables of idle clients), the result naturally extends to our primary algorithm, \algno. To apply this to \algno, the ergodic average $\tilde{x}_i^T$ is simply redefined for each active client $i\in[n]$ as the average of the $x_i^{(r,\ell)}$ coordinates that are explicitly computed, which constitutes a statistically valid random subsequence of the underlying $x_i^t$ trajectory.

\end{document}